\documentclass{article}

    \PassOptionsToPackage{numbers, compress}{natbib}

\usepackage[preprint]{neurips_2022}

\usepackage{comment}

\usepackage[utf8]{inputenc} 
\usepackage[T1]{fontenc}    
\usepackage{hyperref}       
\usepackage{url}            
\usepackage{booktabs}       
\usepackage{amsfonts}       
\usepackage{nicefrac}       
\usepackage{microtype}      
\usepackage{xcolor} 
\usepackage{smile}
\usepackage{comment}
\usepackage{amssymb}
\usepackage{mathtools}     
\usepackage{amsthm,amsmath} 
\usepackage{extarrows}
\usepackage[OT1]{fontenc}
\usepackage{algorithm}
\usepackage{algorithmic}
\usepackage{wrapfig}

\newcounter{module}
\makeatletter
\newenvironment{module}[1][htb]{%
  \let\c@algorithm\c@module
  \renewcommand{\ALG@name}{Module}
  \begin{algorithm}[#1]%
  }{\end{algorithm}
}

\newcounter{pipeline}
\makeatletter
\newenvironment{pipeline}[1][htb]{%
  \let\c@algorithm\c@pipeline
  \renewcommand{\ALG@name}{Pipeline}
  \begin{algorithm}[#1]%
  }{\end{algorithm}
}

\counterwithin{figure}{section}

\title{Bandit Theory and Thompson Sampling-Guided Directed Evolution for Sequence Optimization}

\usepackage{authblk}

\author[1]{Hui Yuan}
\author[2]{Chengzhuo Ni}
\author[3]{Huazheng Wang}
\author[4]{Xuezhou Zhang}
\author[5]{Le Cong}
\author[6]{Csaba Szepesv\'ari}
\author[7]{Mengdi Wang}

\affil[1,2,3,4,7]{Department of Electrical and Computer Engineering\\ Princeton University}
\affil[5]{Department of Pathology and Department of Genetics\\ Stanford University}
\affil[6]{Department of Computing Science\\ University of Alberta \thanks{Authors' emails are: {\texttt {\{huiyuan, cn10, huazheng.wang, xz7392, mengdiw\}@princeton.edu}}, {\texttt {congle@stanford.edu}}, {\texttt {szepesva@ualberta.ca}}.}}

\begin{document}

\def\mw#1{\textcolor{red}{#1}}
\maketitle

\begin{abstract}
Directed Evolution (DE), a landmark wet-lab method originated in 1960s, enables discovery of novel protein designs via evolving a population of candidate sequences. Recent advances in biotechnology has made it possible to collect high-throughput data, allowing the use of machine learning to map out a protein's sequence-to-function relation. There is a growing interest in machine learning-assisted DE for accelerating protein optimization. Yet the theoretical understanding of DE, as well as the use of machine learning in DE, remains limited.
In this paper, we connect DE with the bandit learning theory and make a first attempt to study regret minimization in DE. We propose a Thompson Sampling-guided Directed Evolution (TS-DE) framework for sequence optimization, where the sequence-to-function mapping is unknown and querying a single value is subject to costly and noisy measurements. TS-DE updates a posterior of the function based on collected measurements. It uses a posterior-sampled function estimate to guide the crossover recombination and mutation steps in DE. In the case of a linear model, we show that TS-DE enjoys a Bayesian regret of order $\tilde O(d^{2}\sqrt{MT})$\footnote{$\tilde O(\cdot)$ ignores the logarithmic terms.}, where $d$ is feature dimension, $M$ is population size and $T$ is number of rounds. This regret bound is nearly optimal, confirming that bandit learning can provably accelerate DE. It may have implications for more general sequence optimization and evolutionary algorithms. 
\end{abstract}

\section{Introduction}

Protein engineering means to design a nucleic acids sequence for maximizing a utility function that measures certain fitness or biochemical/enzymatic properties, i.e., stability, binding affinity, or catalytic activity. Due to the combinatorial sequence space and lack of knowledge about the sequence-to-function map, engineering and identifying optimal protein designs were a quite daunting task. It is only until recently that synthesis of nucleic acid sequences and measurement of protein function became reasonably scalable \citep{packer2015methods, yang2019machine}, allowing rational optimization or directed evolution of protein designs. Nonetheless, because of the complex landscape of protein functions and the bottleneck of wet-lab experimentation, this remains a very difficult problem.

Directed evolution (DE), one of the top molecular technology breakthrough in the past century, demonstrate human's ability to engineer proteins at will. 
DE is a method for exploring new protein designs with properties of interest and maximal utility, by mimicking the natural evolution. It works by artificially evolving a population of variants, via mutation and recombination, while constantly selecting high-potential variants \citep{chen1991enzyme, chen1993tuning, kuchner1997directed, hibbert2005directed, turner2009directed, packer2015methods}. The development of directed evolution methods was honored in 2018 with the awarding of the Nobel Prize in Chemistry to Frances Arnold for evolution of enzymes, and George Smith and Gregory Winter for phage display \citep{arnold1998design, smith1997phage, winter1994making}. See Figure \ref{fig:R_M} for illustrations of mutation and crossover recombination.  

\begin{wrapfigure}{r}{0.5\textwidth}
  \begin{center}
    \includegraphics[width=0.45\textwidth]{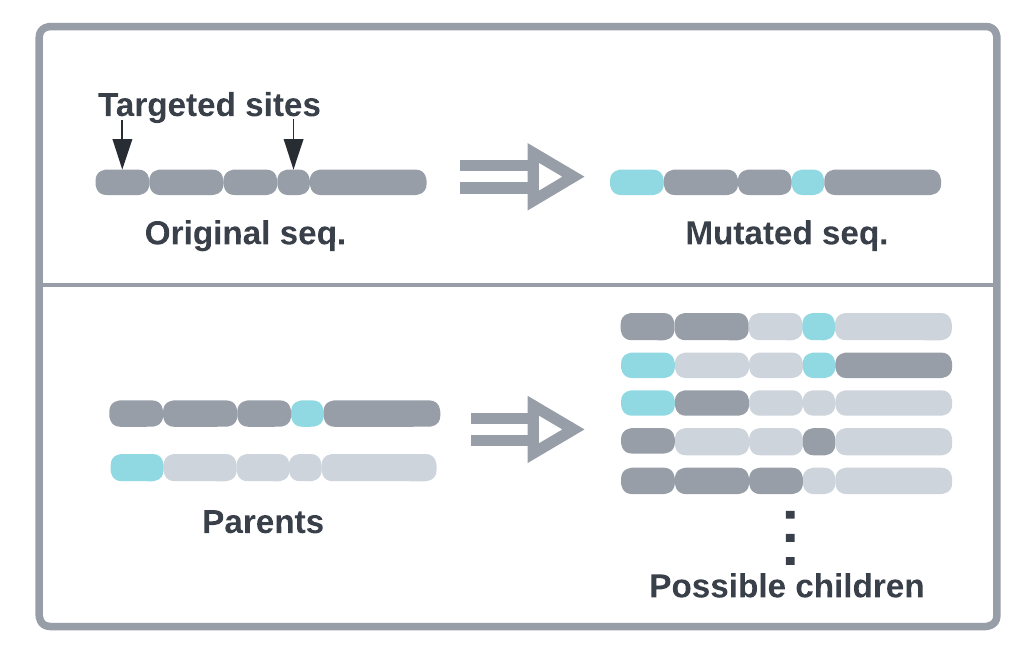}
  \end{center}\vspace{-5pt}
  \caption{Illustration of mutation and crossover recombination.
Mutating a sequence means to replace a targeted or random entry (site) by a random or designated value.
Recombination involves two or multiple sequences. For example, parent sequences can crossover, exchange subsequences and generate children. 
}
  \label{fig:R_M}
\end{wrapfigure}

However, DE often remains expensive and time-consuming. The major considerations  center on cost and data quality. First, the ability to synthesize and mutate new biological sequences have been exponentially improved thanks to synthetic chemistry advances. Second, given a population of sequences $S$, selecting and identifying the set of optimal sequences is straightforward, using low-cost parallel sequencing which works well with pooled selection assays. Third, using pooled measurement to evaluate the average value of protein function (mean fitness) over a population $S$ is generally easy, as such bulk measurements is low-cost and high-quality. Finally, querying $f(x)$ for a given $x$ is often expensive, and the cost adds up quickly if many queries are needed. It can be desirable to perform this procedure in small-scale batches to optimize time and resource consumption.


Such difficulties have motivated scientists to apply machine learning approaches to accelerate DE, beginning with \citet{fox2003optimizing} and followed by many. Recent development of directed evolution have increasingly utilized \textit{in silico} exploration and machine learning beyond experimental approaches \citep{yang2019machine, fannjiang2020autofocused, doppa2021adaptive, shin2021protein, freschlin2022machine}.  While these attempts have proved to be successful in simulation and sometimes in real experiments, little is known about the statistical theory of DE. 

In this paper, a primary objective is to bridge the directed evolution process with bandit learning theory. In particular, we want to express machine learning-assisted DE as a bandit optimization process, with a theoretical justification. Further, we aim to understand how a machine learning model, as simple as linear, can accelerate DE and reduce the overall cost of evaluation. Specifically, we propose a Bayesian bandit model for DE, namely the Thompson Sampling-guided Directed Evolution framework, which combines posterior model sampling with directed mutation and recombination. The theoretical analysis shows that the crossover selection mimics an optimization iteration, and the optimization progress is proportional to a level of population diversity. In the case of the linear model, we establish a Bayesian regret bound $\tilde O(d^{2}\sqrt{MT})$ that depends polynomially on feature dimension $d$ and optimally in batch size $M$ and time steps $T$. We finally harmonize our theoretical analysis with a set of simulation and real-world experimental data.

\section{Related work}

Our analysis is related to the theoretical literature on evolutionary algorithms and linear bandits.
\paragraph{Evolutionary algorithm.}

The success of DE motivated a large body of works on evolutionary algorithms for optimization.  Evolutionary algorithm (EA) \citep{back1996evolutionary} is a large class of randomized optimization algorithms, based on the heuristic of mimicking natural evolution.  Despite many variants, a typical EA usually maintains a population of solutions and improves the solutions by alternating between reproduction step which produces new offspring solutions, and selection step where solutions are evaluated by the objective function and only the good ones are saved to the next round. 
Theoretical understandings of EA are focusing on specific EAs, among which the most well-studied setting is $(1+1)$-EA, with parent population size and offspring population size are both 1 to optimize linear objective function on the Boolean space $\{0,1\}^d$, see \citep{droste2002analysis, he2004study, jagerskupper2008blend, jagerskupper2011combining, lehre2012black, witt2013tight}. EA analysis focuses on optimization and reducing the running time instead of minimizing total regret as in bandit theory. There are other results on population based EAs, such as $(1+\lambda)$-EA \citep{doerr2015optimizing, giessen2016optimal}, $(\mu+1)$-EA \citep{witt2006runtime} and the most general $(\mu+\lambda)$-EA, where $\mu$ and $\lambda$ represent the parent population size and the offspring population size respectively. 
However, this group of works only adopted mutation. The understanding of the role played by recombination in evolutionary algorithms was left as blank in the $(\mu+\lambda)$-EA framework, while our paper provides a population-based regret minimization analysis with both mutation and recombination. 

There are a few works \citep{jansen2002analysis, jansen2005real,  watson2007building, kotzing2011crossover} studying EAs with recombination (which are also called genetic algorithms (GAs)). 
However, their algorithms and analysis are tailored to artificial test objectives and the results are not able to generalize even to linear objectives. Recently, the running time analysis of some natural EAs with recombination has been conducted \citep{oliveto2015improved, oliveto2020tight}, but still their results are constrained under specific objectives such as $\operatorname{ONEMAX}$ and $\operatorname{JUMP}$. We refer readers to the book by \citep{zhou2019evolutionary} for a more comprehensive review of EA.

\paragraph{Linear bandits.}
Bandit is a powerful framework formulating the sequential decision making process under uncertainties. Under this framework, linear bandits is a central and fruitful branch where in each round a learner makes her decision and receives a noisy reward with its mean value modelled by a linear function of the decision, aiming to maximize her total reward (or minimize total regret equivalently) over multiple rounds  \citep{Auer02,li2010contextual,abbasi2011improved}. In the same spirit, the process evolving a population of genetic sequences to maximize a linear utility over the evolution trajectory, while getting access to noisy utility values through evaluating sequences along the way, can be mathematically formulated from the perspective of linear bandits.
One of the main solution in linear bandits is the upper confidence bound-based (UCB) strategy represented by LinUCB \citep{li2010contextual}, where the learner makes decision according to upper confidence bounds of the estimated reward and the accumulated regret is proven to be $\tilde O \left( d \sqrt{T} \right)$. A similar strategy is optimism in the face of uncertainty (OFU) principle in \citet{abbasi2011improved}. The other approach is the Thompson Sampling (TS) strategy, which randomizes actions on the basis of their probabilities to be optimal. \citet{russo2014learning} proved the Bayesian regret of TS algorithm is also of order $\tilde O \left( d \sqrt{T} \right)$. And there are more results on the regret of TS(-like) algorithms solving linear bandits in the frequentist view \citep{agrawal2013thompson, abeille2017linear, hamidi2020worst}. TS is also powerful beyond the scope of linear bandits, such as contextual bandits \citep{agrawal2013thompson}, reinforcement learning \citep{zhang2021feel}.
We also refer readers to the book by \citep{lattimore2020bandit} for a delicate review of bandit theory.

\paragraph{Remark.} It is important to note that our problem is {\it not} a multi-armed bandit problem. In bandits, one can choose actions freely from the full action set. However, in biological experiments, it is expensive to synthesize a new protein design sequence out of thin air. Instead, mutation and recombination are used to generate new designs easily at a low cost. Thus our algorithm can only guide the selection step in the DE process. Its regret is not directly comparable with the regret of multi-arm bandits. To the best of our knowledge, this is the first work that studies the bandit theory and regret bound of mutation and recombination-enabled DE.

\section{Bandit model for directed evolution}

\subsection{Process overview}

We illustrate the Thompson Sampling-guided Directed Evolution (TS-DE) process in Figure \ref{fig:Bandit_DE}. A population $S_t$ at time $t$ consists of $M$ candidate sequences. It evolves via mutation, crossover recombination, selection, and function evaluation to the next generation $S_{t+1}$. The mutation and crossover selection are guided using a learnt function $f_{\tilde \theta_t}$, in order to filter out unwanted candidates and keep only a small batch for costly evaluation. Collected data are fed into a Thompson Sampling module for posterior update of $f_{\tilde \theta_t}$. 
Full details of the mutation, crossover selection, and Thompson Sampling modules will be given in Section 4.

\begin{figure}[h]
    \centering
    \vspace{-2mm}
    \includegraphics[scale=0.6]{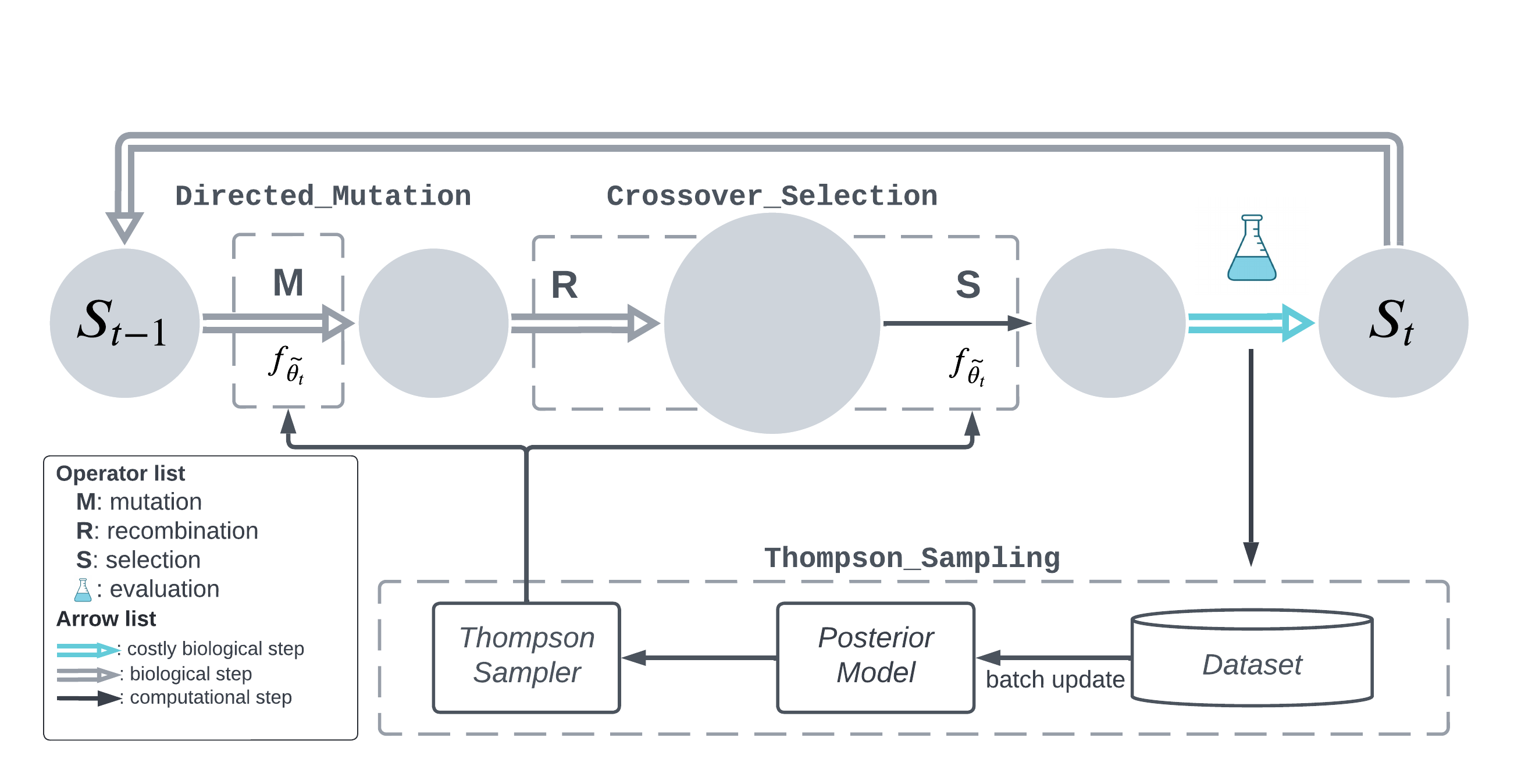}
    \vspace{-2mm}
    \caption{\textbf{Thompson sampling-guided directed evolution}}
    \label{fig:Bandit_DE}
    \vspace{-2mm}
\end{figure}


\subsection{Motif feature, utility model, recombination and mutation operators}

A genetic sequence comprises of functional motifs, i.e., functional subsequences that may encode particular features of protein, also known as protein motifs \citep{ofer2015profet, tubiana2019learning, cunningham2020biophysical}. Such genetic motifs are known to be ``evolutionarily conserved", in the sense that they tend to evolve as units, under mutation and recombination.


Suppose a genetic sequence $seq$ is made up of $d$ genetic motifs, given by 
$
    seq = (seq_{(1)},seq_{(2)}, \cdots, seq_{(d)}).
$
Machine learning models for protein utility prediction are often based on motif features \citep{wittmann2021advances, cunningham2020biophysical, ofer2021language}. 
Let $\mathcal{X}$ be the space of genetic sequences of interest.
We assume that a binary motif feature map is given, defined as follows.

\begin{definition}[Binary Motif Feature Embedding]
 Let $\phi$ be the genetic motif feature map given by:
\begin{equation}
    \phi: \mathcal{X} \to \{0,1\}^{d}, \quad \phi(seq): = (\phi_1(seq_{(1)}), \cdots, \phi_d(seq_{(d)}))
\end{equation}
such that at each dimension $i$, $\phi_i(x_{(i)})$ is a binary feature of motif $x_{(i)}$. 
\end{definition}

The binary motif feature provides a minimalist abstraction for evolutionary processes where $0,1$ correspond to favorable and nonfavorable directions, respectively, for each motif. 
Theoretical analysis for evolutionary optimization algorithms made the same assumption and viewed binary sequence optimization as a fundamental problem  \citep{droste2002analysis, he2004study, jagerskupper2008blend, jagerskupper2011combining, lehre2012black, witt2013tight}.

Since a protein function is largely determined by its motif, it is common to model the protein utility $f: \mathcal{X} \to \mathbb{R}$ as a function of motif features, i.e., 
$
    f(seq) := f_{\theta^{\star}}(x),   x=\phi(seq), \forall seq \in \mathcal{X},
$
under a parameterization by $\theta^{\star}$ \citep{fox2003optimizing, yang2019machine, shin2021protein, freschlin2022machine}.

In this work, we study the most elementary Bayesian linear model, where $f$ is a linear model parameterized by $\theta^*$ with a Gaussian prior, given as follows.

\begin{assumption}
\label{asmp_linearf}
(Linear Bayesian Utility Model) Assume the utility $f_{\theta^{\star}}$ is a linear function parameterized by $\theta^{\star} \in \mathbb{R}^{d}$, which is sampled from a Gaussian prior, i.e.
\begin{equation}
    f_{\theta^{\star}} (x) = \langle \theta^{\star}, x\rangle, \qquad
    \theta^{\star} \sim \mathcal{N}(\mathbf{0}, \lambda^{-1} \mathbf{I} ), \quad \lambda > 0.
\end{equation}
\end{assumption}

Since motifs tend to mutate and recombine with one another in units, it is often sufficient to focus on recombination and mutation on the motif level, rather than on the entry level.
Further, recombination that breaks a motif often result in insignificant low-fitness descendants. 
Therefore, it suffices to focus on motif-level directed evolution for simplicity of presentation and theory. For theoretical simplicity, we define recombination and mutation operators \textbf{on the motif level}:

\begin{definition}[Directed Mutation Operator]
\label{opt:mut}
Let $x$ be the motif feature sequence, $\mathcal{I} 
\subset [d]$ be a collection of targeted sites and $\mu \in (0,1)$ be a mutation rate. The mutation operator $\texttt{Mut}(x, \mathcal{I}, \mu)$ generates a sequence $x^{\prime}$ such that while for $\forall j \not \in \mathcal{I}, x_j^{\prime} = x_j$, for $\forall i \in \mathcal{I}, x_i^{\prime}$ is independently induced to be
        \begin{equation}
        \label{equ:d_mutation}
        \left\{
            \begin{array}{cc}
            x^{\prime}_i \sim \operatorname{unif}(\{0,1\}), &\text{w.p. } \mu,\\
            x^{\prime}_i = x_i, &\text{otherwise.}
            \end{array}
        \right.
       \end{equation}
\end{definition}


\begin{definition}[Recombination Operator]
\label{opt:rcb}
Let $x,y$ be the motif features associated with two parental genetic sequences. The recombination operator $\texttt{Rcb}(x, y)$ generates a child sequence $z$ such that $z_i$'s are independent and 
\begin{equation}
\label{def:rcb_dist} 
    z_i = \left\{ 
    \begin{array}{cc}
         x_i & \text{w.p. } \frac{1}{2}\\
         y_i & \text{w.p. } \frac{1}{2}
    \end{array},
    \right. \forall i \in [d].
\end{equation}
\end{definition}

We remark that Definitions $\ref{opt:mut},\ref{opt:rcb}$ are {\it mathematical simplifications} of their real-world counterparts. In real world, mutation and recombination can take various forms depending on the context. In our analysis, we define them in a minimalist-style to keep theory generalizable and interpretable.

\subsection{Regret minimization problem formulation}

Evaluating the protein function for a design sequence $x$ is a most costly and time-consuming step in protein engineering. In the DE process, we consider that regret is incurred only when sequences are evaluated. We also assume that each evaluation is subject to a Gaussian noise with known variance.


\begin{assumption}
\label{asmp_noise}
(Noisy Feedback) Upon querying the utility of $x$, we get an independent noisy evaluation given by
\begin{equation}
\label{equ:noise}
    u(x) \sim \mathcal{N}( f_{\theta^{\star}} (x), \sigma^2).
\end{equation}
\end{assumption}

Our goal is to maximize the Bayesian regret, i.e., the  cumulative sum of optimality gaps between evaluated sequences and the optimal.
\begin{definition}[Bayesian Regret]
\label{def:bayes_rgt}
Denote by $f_{\theta^{\star}}(x^{\star})$ the optimal utility value over $\mathcal{X}$, $\{x_{t,i}\}_{i=1}^{M}$ are the evaluated individuals in each iteration. Throughout $T$ iteration, the accumulated regret is defined as
\begin{equation*}
    \operatorname{BayesRGT}(T, M) = \mathbb{E}\left[\sum_{t=1}^T \sum_{i=1}^M (f_{\theta^{\star}}(x^{\star}) - f_{\theta^{\star}}(x_{t,i}))\right],
\end{equation*}
where $M$ is number of sequences selected for evaluation per timestep, and $\mathbb{E}$ is taken over the prior of $\theta^{\star}$ and all randomness in the DE process.
\end{definition}

\section{Thompson Sampling-guided directed evolution (TS-DE) }

We restate our goal as to direct a population of genetic sequence to evolve towards higher utility value, until its population-average converges to the optimum $f_{\theta^{\star}}(x^{\star})$. Our knowledge of $f$ is to be learned from noisy evaluations of selected sequences along the way. In this section, by integrating the biological technique -  directed evolution - with Thompson Sampling, a Bayesian bandit method, we propose the Thompson Sampling-guided Directed Evolution algorithm (TS-DE) as shown in Alg.\ref{alg:one-level(PS)}, where in each round Thompson sampling gives an estimate of $\theta^{\star}$, based on which key operators of DE: mutation, recombination and selection are implemented.

\subsection{Crossover-then-selection and directed mutation}

Pairwise crossover is a most common type of recombination in natural evolution. Let $x,y$ be a random pair of parents, and let $z=\texttt{Rcb}(x,y)$ be a child. If given a utility function $f$, we select $z$ only if the child performs better than the parents' average. Module \ref{mod:rcb&slc} formulates this procedure.

\begin{module}
\caption{ $\texttt{Crossover\_Selection}(f,S)$}
\label{mod:rcb&slc}
\begin{algorithmic}[1]
    \STATE \textbf{Inputs:} utility function $f(x) = \langle \theta, x \rangle$, a population of sequences $S$
    \STATE \textbf{Initialization:} $S' \leftarrow \emptyset$ 
    \WHILE{$|S^{\prime}| < |S|$}
        \STATE Sample $x$ and $y$ from $S$ uniformly with replacement.\\
        \STATE \textbf{Recombination: }
         $z \leftarrow \texttt{Rcb}(x, y)$ (Definition \ref{opt:rcb}).\\
        \STATE \textbf{Selection}: $S^{\prime}\leftarrow S^{\prime}\cup\{x'\}$ if $f(z) \geq \frac{f(x)+ f(y)}{2}$. 
    \ENDWHILE
    \STATE \textbf{Output:} $S^{\prime}$ 
\end{algorithmic}
\end{module}


Next we turn to designing the strategy for adding directed mutation under a given $f$ as guidance and propose Module \ref{mod:dire_mutation}. An ideal mutation will diversify the population while preserving its fitness level as much as possible. So we add directed mutation to sites where the single site fitness over the population is less than of a uniformly distributed sequence. Formally, we only add mutation to site $i$ if  $\frac{1}{M} \sum_{x\in S} \theta_i \cdot x_i \leq  \theta_i \cdot \bar x_i $,
where $\bar x_i$ is the mean of uniformly random $x_i$.

\begin{module}
\caption{$\texttt{Directed\_Mutation}(f, S, \mu)$}
\label{mod:dire_mutation}
\begin{algorithmic}[1]
    \STATE \textbf{Inputs:} utility function $f(x)= \langle \theta, x \rangle$, a population of sequences $S$, mutation rate $\mu$
    
    \STATE \textbf{Initialization:} $\mathcal{I}\leftarrow\emptyset,\mathcal{S'}\leftarrow\emptyset$
        \FOR{$ i \in [d]$}
             \IF{$\frac{1}{M} \sum_{x\in S} \theta_i  \cdot x_i \leq  \theta_i \cdot \bar x_i$}
             \STATE $\mathcal{I} \leftarrow \mathcal{I}\cup\{i\}$.
             \ENDIF
        \ENDFOR
    \STATE \textbf{Directed Mutation:}    $x^{\prime} = \texttt{Mut}(x, \mathcal{I}, \mu)$ (Definition \ref{opt:mut}) and $S^{\prime}\leftarrow S^{\prime}\cup\{x^{\prime}\}$ for all $x\in S$.  
    \STATE \textbf{Output:} $S^{\prime}$ 
\end{algorithmic}
\end{module}

\subsection{Full algorithm}

Finally, we are ready to combine all modules and state the full algorithm in Algorithm \ref{TSDE}. At each time step $t$, a posterior distribution is first computed using the data collected in history. Then we sample a $\tilde{\theta}_t$ from the posterior and do the corresponding directed mutation and crossover selection using this sampled weight, and augment the dataset for the next iteration with the measurements of resulting new population. The procedure is repeated until the time limit $T$ is reached.

\begin{algorithm}[htb!]
\caption{Thompson Sampling-Guided Directed Evolution (TS-DE)\label{TSDE}}
\label{alg:one-level(PS)}
\begin{algorithmic}[1]
    \STATE \textbf{Inputs:} number of rounds $T$, initial population $S_0 = \{x_{0,i}\}_{i = 1}^{M}$ of size $M$, mutation rate $\mu$, $\sigma$ 
    \STATE \textbf{Initialization:} dataset $D_0 \leftarrow \emptyset$,  $\Phi_{t-1} = 0$, $U_0 = 0$
    \FOR{ $t=1$ to $T$}	
    \STATE \textbf{Posterior update} 
    \begin{align}
    \label{equ:RLS}
        V_{t} &= \frac{1}{\sigma^2} \Phi^{\top}_{t-1} \Phi_{t-1} + \lambda I,\qquad
        \hat \theta_{t} = \frac{1}{\sigma^2} V^{-1}_{t} \Phi^{\top}_{t-1} U_{t-1}. 
    \end{align}
    \STATE \textbf{Thompson Sampling}  $ \quad \tilde\theta_t \sim \mathcal{N}(\hat \theta_{t}, V^{-1}_{t})$. 
    
    \STATE $S_{t-1}^{\prime} = \texttt{Directed\_Mutation}(f_{\tilde \theta_t}, S_{t-1}, \mu)$ (Module \ref{mod:dire_mutation}). 
    
    \STATE $S_t = \texttt{Crossover\_Selection}(f_{\tilde \theta_t}, S_{t-1}^{\prime})$ (Module \ref{mod:rcb&slc}).
   
    \STATE \textbf{Evaluation and data collection} Evaluate the utilities of all individuals in $S_t$ and $D_{t} \leftarrow D_{t-1} \cup \{x_{t,i}, u(x_{t,i})\}_{i = 1}^{M}$. Update $\Phi^{\top}_t \leftarrow \left(\Phi^{\top}_{t-1}, x_{t,1}, \cdots, x_{t,M} \right)$, $U_{t} \leftarrow \left( U_{t-1}^{\top}, u(x_{t,1}), \cdots, u(x_{t,M}) \right)^{\top}$. 
    \STATE $t \leftarrow t+1$.
    \ENDFOR
\end{algorithmic}
\end{algorithm}


\section{Main results}

In this section, we analyze the performance of TS-DE (Algorithm \ref{TSDE}). We will show that the crossover selection module essentially mimics an optimization iteration that strictly improves the population's fitness along the designated direction. By using a Bayesian regret analysis, we show the DE modules, when combined with posterior sampling, can effectively optimize towards the best protein design while learning $\theta^{\star}$. 

\subsection{Crossover selection as an optimization iteration}
Let $f$ by any utility function, and let $F(S):=\operatorname{avg}_{x\in S} f(x)$ denote the population average utility. 
Our first result states an ascent property showing that \texttt{Crossover\_Selection} strictly improves the population average. 

\begin{theorem}[Ascent Property of Recombination-then-Selection]
\label{thm:ascent}Let $f(x) = \langle \theta, x \rangle$ and let $S$ be a set of sequences. 
Let $S^{\prime} = \texttt{Crossover\_Selection}(f,S)$, then it satisfies
\begin{equation}
\label{equ:ascent}
    \mathbb{E}\left[ F(S^{\prime}) \right]  \geq F(S) + \frac{\mathbb{E}_{x,y} \left[ \| \theta \cdot \left( x - y \right)\| \right]}{2 \sqrt{2}} \geq F(S) + \frac{1}{\sqrt{2d}} \sum_i \left| \theta_i \right| \textrm{Var}_i(S),
\end{equation}
where $\textrm{Var}_i(S)$ denotes the variance of $x_i$ when $x$ is uniformly sampled from $S$.
\end{theorem}

\begin{wrapfigure}{r}{0.35\textwidth}
  \begin{center}\vspace{-7pt}
    \includegraphics[width=0.33\textwidth]{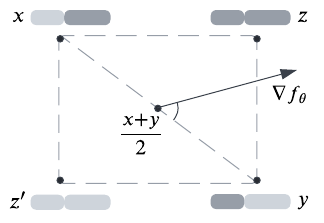}
  \end{center}
  \caption{{Ascent property of crossover recombination}}
  \label{fig:ascent}\vspace{-5pt}
\end{wrapfigure}
\paragraph{Proof sketch.}
See Figure \ref{fig:ascent} for illustration.
Given $x$ and $y$, $z = \texttt{Rcb}(x, y)$ can be represented by
$
    z = \frac{x + y}{2} + \frac{x - y}{2} \cdot e
$,
where the $\cdot$ denotes the entrywise multiplication between two vectors and $e= (e_i, \cdots, e_d)$ with $e_i$'s being independent Rademacher variables. Then $f(z)$ equals
$
    \frac{f(x) + f(y)}{2} + \frac{1}{2} \sum_{i=1}^{d} \theta_i \left( x_i - y_i \right) e_i .
$
After the selection step, the expected amount by which $f(z)$ exceeds its parents' average is at least 
$
    \frac{1}{2} \mathbb{E} \left[  \left|\sum_{i=1}^{d} \theta_i \left( x_i - y_i \right) e_i  \right| \right],
$
which has a tight lower bound of $\frac{1}{2 \sqrt{2}} \| \theta \cdot \left( x - y \right)\|$ according to \citet{haagerup1981best}. The full proof is given in Appendix \ref{app-ascent-general}. \hfill $\blacksquare$


\paragraph{Remark on diversity.} Analysis above reveals an intriguing observation: the optimization progress of \texttt{Crossover\_Selection} scales linearly with $\sum_i \theta_i \textrm{Var}_i(S)$, i.e., sum of per-motif variances across population $S$. It measures a level of ``diversity" of $S$ with respect to direction $\theta$. More diverse population would enjoy larger progress from crossover selection. This observation is consistent with the natural evolution theory that diversity is key to the adaptability of a population to cope with evolving environment where fitness traits are essential \citep{whittaker1972evolution}.

\subsection{Regret bound of TS-DE}

Our main result is a Bayesian regret bound for TS-DE. Recall from Definition \ref{def:bayes_rgt} that $\operatorname{BayesRGT}(T, M) = \mathbb{E}[\sum_{t=1}^T \sum_{i=1}^M (f_{\theta^{\star}}(x^{\star}) - f_{\theta^{\star}}(x_{t,i}))]$. 

\begin{theorem}
\label{main_thm}
Under Assumption \ref{asmp_linearf} and  \ref{asmp_noise}, when the population size is sufficient s.t. $M = \Omega \left(  \frac{\log(d T)}{\mu^2} \right)$, Alg.$\ref{alg:one-level(PS)}$ admits its Bayesian regret s.t.
\begin{equation}
    \operatorname{BayesRGT}(T, M) = \tilde O \left( \frac{d}{\mu \sqrt{\lambda}} \cdot d \sqrt{MT} \right).
\end{equation}
If we let $\lambda=1,\mu =1/2, \sigma^2=1$, the Bayesian regret simplifies to $\tilde O(d^2 \sqrt{MT})$.
\end{theorem}

\paragraph{Remark on regret bound.}
Regret bound of Theorem \ref{main_thm} is optimal in $M,T$. For comparison, the Bayesian regret of Gaussian linear model is $\tilde O(d\sqrt{T})$ \citep{kalkanli2020improved}, also in contextual linear bandit with batch update, the optimal regret is $\tilde O(d \sqrt{MT})$ \cite{han2020sequential}. Our TS-DS regret has two extra factors of $\sqrt{d}$. One $\sqrt{d}$ is due to that the $l_2$ norm of our feature vectors are $\sqrt{d}$, while linear bandit theory often assumes feature to have norm $1$. Another factor of $\sqrt{d}$ is due to the evolutionary nature of DE, i.e., TS-DE is not allowed to any possible action but have to select those from the evolving population.

\subsection{Proof sketch}
Denote by $x^{\star}$ and $x_t^{\star}$ the maximums of $f_{\theta^{\star}}$ and $f_{\tilde \theta_t}$. Denote by $F^{\star}_t: = f_{\tilde \theta_t}(x_t^{\star})$  the maximum value of $f_{\tilde \theta_t}$ and denote by $F_t(S)$ the average value of $f_{\tilde \theta_t}$ over set $S$.

\textbf{Step 1: Regret decomposition.}
With expectation taken over all 
stochasticity, posterior sampling guarantees 
$
\operatorname{BayesRGT}(T, M) = \sum_{t=1}^T \sum_{i=1}^M \mathbb{E} \left[ f_{\tilde \theta_t}(x_t^{\star}) - f_{\theta^{\star}}(x_{t,i}) \right]
$
since conditioned on data $D_{t-1}$, $f_{\theta^{\star}}(x^{\star})$ and $f_{\tilde \theta_t}(x_t^{\star})$ are identically distributed.
Then by breaking $f_{\tilde \theta_t}(x_t^{\star}) - f_{\theta^{\star}}(x_{t,i})$ down to the sum of $ f_{\tilde \theta_t}(x_t^{\star}) - f_{\tilde \theta_t}(x_{t,i})$ and $ f_{\tilde \theta_t}(x_{t,i}) - f_{\theta^{\star}}(x_{t,i})$, we decompose the total regret into 
\begin{equation}
    \operatorname{BayesRGT}(T, M) = M \cdot \underbrace{ \mathbb{E}  \left[ \sum_{t=1}^T \left( F^{\star}_t - F_t(S_t) \right) \right]}_{H_1} + \underbrace{\mathbb{E} \left[\sum_{t=1}^T \sum_{i=1}^M \langle \tilde\theta_t -\theta^{\star}, x_{t,i} \rangle \right]}_{H_2}.
\end{equation}

\textbf{Step 2: Bounding $H_1$ using linear convergence.}
$H_1$ is the accumulated optimization error under a time-varying objective $f_{\tilde\theta_t}$.
After calling $S_{t-1}^{\prime} = \texttt{Directed\_Mutation}(f_{\tilde \theta_{t}}, S_{t-1}, \mu)$ and $S_{t} = \texttt{Crossover\_Selection}(f_{\tilde \theta_t}, S_{t-1}^{\prime})$ at step $t$, the ascent property (\ref{equ:ascent}) together with property of the mutation module yields a linear convergence towards $F^{\star}_{t}$, i.e.,
$
     \mathbb{E} \left[ F^{\star}_{t} - F_{t}(S_{t}) \mid  S_{t-1} ,\tilde \theta_t\right] \leq \gamma(F^{\star}_{t} - F_{t}(S_{t-1}))
$
with a modulus of contraction $\gamma\in(0,1)$ s.t. $\frac{1}{1 - \gamma} = O \left( \frac{\sqrt{d}}{\mu}\right)$.
It follows that
\begin{align*}
    F^{\star}_{t+1} - F_{t+1}(S_{t+1}) &\leq  \gamma \left[F^{\star}_{t} - F_{t}(S_{t}) \right] + \hbox{error terms}
    + e_{t+1},
\end{align*}
where $e_{t+1}$ is a martigale difference.
Applying the above recursively
to $H_1$, we get  $H_1\leq $ 
\begin{equation*}
   \underbrace{ \frac{1}{1-\gamma} \cdot \mathbb{E} \left[  F^{\star}_{1} - F_{1}(S_{0})\right]}_{O( \frac{1}{1-\gamma})} + \underbrace{\mathbb{E} \left[  \sum_{k = 2}^{T}  \gamma^{T-k+1} F^{\star}_{k} -\gamma^{T-1} F^{\star}_{1} \right]}_{O( \frac{1}{1-\gamma})} + \underbrace{ \mathbb{E} \left[ \sum_{t = 1}^T \sum^{t-1}_{k=1}\gamma^{t-k} \left( F_{k}(S_{k}) - F_{k+1}(S_{k})\right)\right]}_{\kappa},
\end{equation*}
which is dominated by term $\kappa$ and $M \cdot \kappa \leq \frac{1}{1-\gamma} \cdot  \sum_{t = 1}^{T-1} \sum_{i = 1}^{M}  \left|\langle \tilde \theta_{t} -  \tilde \theta_{t+1}, x_{t,i} \rangle \right| = O \left( \frac{1}{1-\gamma} H_2 \right)$.

\textbf{Step 3: Bounding $H_2$.} $H_2$ is the accumulated prediction error of $\tilde\theta_t$,  which is a classic term to bound in bandit literature and is of $\tilde O \left( d^{1.5} \sqrt{MT}  \right)$ by using a batched self-normalization bound.
$\hfill \blacksquare$

\section{Experiments}

\subsection{Simulation}

We test the TS-DE by simulating the evolution of a population of sequences in $\{0,1\}^d$. We set the initial population to be all zeros, and set $\lambda = 1$, $\sigma = 1$. 


\textbf{Regret and convergence results.} Figure \ref{fig:Rgt_f} shows the regret curves and learning curves of TS-DE, with comparison to basic DE. In the left panel of Figure \ref{fig:Rgt_f}, we plot the population-averaged Bayesian regret of TS-DE with various values of $M$, where $d = 10$, $T = 100$ and $\mu = 0.8$. These results confirm our sublinear regret bounds. In the right panel of Figure \ref{fig:Rgt_f}, we tested TS-DE using various mutation rates, and compared them with a basic DE approach \footnote{The basic DE approach does not employ any function estimate. It does random mutation with a predefined mutation rate and random crossover recombination. It evaluates every candidate sequence and uses the noisy feedback in replace of $f_{\tilde \theta}$ for selection.}. 
The comparison shows that TS-DE converges significantly faster, while the convergence of DE is much slower and very sensitive to mutation scheduling.

\begin{figure}[t]
    \centering
    \vspace{-2mm}
    \includegraphics[width=\linewidth]{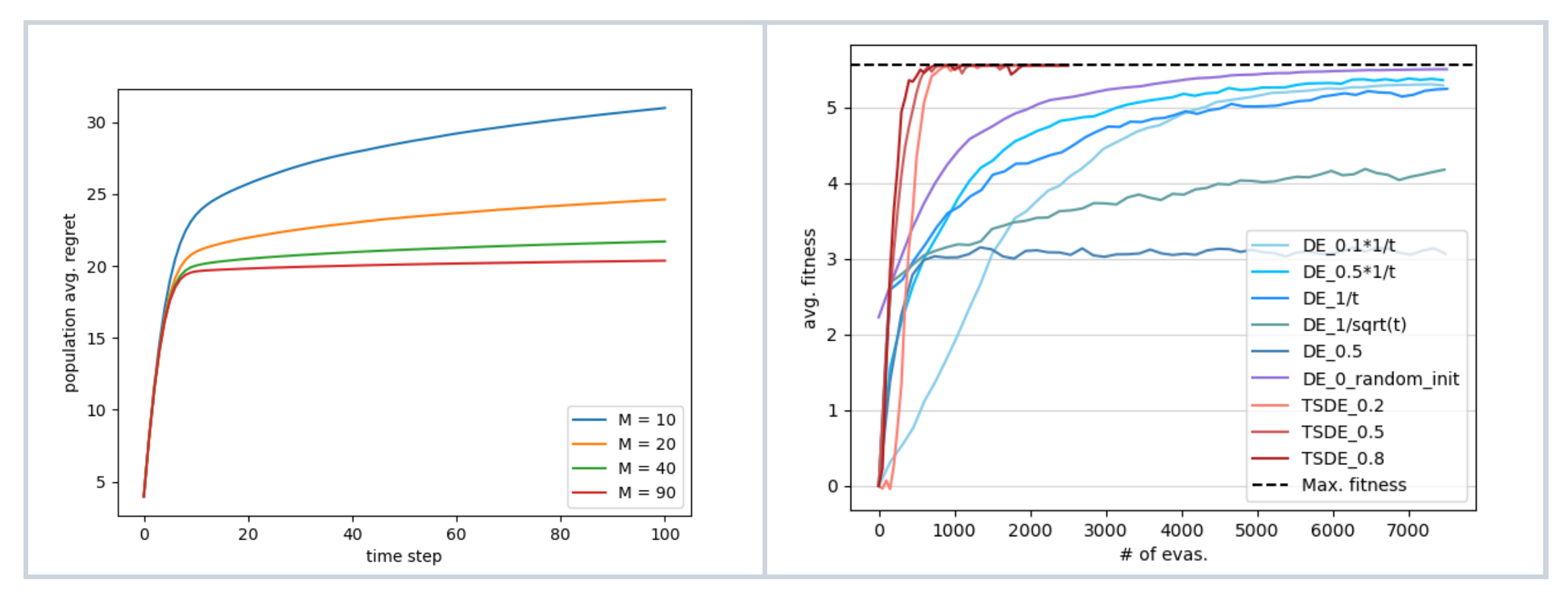}
    \vspace{-2mm}
    \caption{\textbf{Regret and fitness curves of TS-DE during evolution.} Left: Population-averaged regret with varying population sizes $M$. Each curve is averaged over 100 trials. Right: Fitness curves of TS-DE with varying values of $\mu$, compared with basic DE with varying mutation rates. (The purple curve plots basic DE without mutation, we modified the initial population to be uniformly distributed in this case to make it non-trivial.) \label{fig:Rgt_f}}
    \vspace{-2mm}
\end{figure}

\textbf{Visualizing the evolution of a population.} 
 We visualize the evolution trajectory of population $S_t$ in one run of TS-DE, with $d = 40$, $M = 20$ and $\mu = 0.1$. In the left panel of Fig.\ref{fig:S_f_evo}, we visualize the evolving high-dimensional population $S_t$ by mapping them to 2D (via PCA and KDE density contour plot). In the right panel of Fig.\ref{fig:S_f_evo}, we plot the fitness distribution of each $S_t$. These plots illustrate how TS-DE balances the exploration-exploitation trade-off: It guides $S_t$ to ``diversify" initially and then quickly approach and concentrate around a maximal solution.

\begin{figure}[t]
    \centering
    \includegraphics[width=\linewidth]{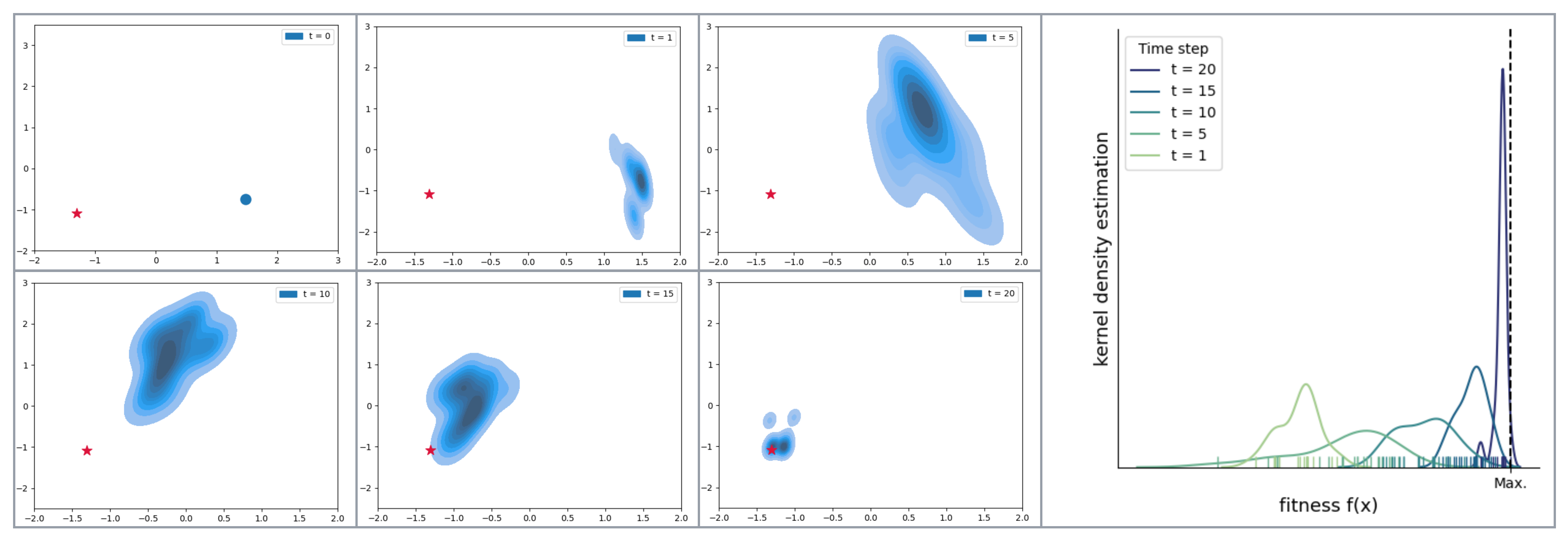}
    
     \vspace{-2mm}
    \caption{\textbf{Evolving population of TS-DE and fitness levels.}  Left panels: Visualization of population evolution projected in 2D shown, taken at 6 snapshots. Right panel: The population's fitness distribution shifts towards optimal during evolution. \textcolor{red}{$\star$} denotes the optimal solution.  
    \label{fig:S_f_evo}
    }
    \vspace{-2mm}
\end{figure}

\subsection{Real-world experiment validation} 

Having demonstrated our approach with simulations, we use real-world experiments to showcase the validity and generalizability of our method. The TS-DE method is adapted to work with real-world motif features (continuous-valued instead of binary), linear model and multiple rounds of wet-lab experiments for optimizing a CRISPR design sequence. Our approach together with high-throughput experiment identified a high-performing sequence with 30+ fold improvement in efficiency. 
Notably, the optimized CRISPR designs generated by our DE approach is part of another manuscript (in press at a biological journal, Molecular Cell), demonstrating real-world utility of our method.
We postpone more details about this real-world validation to Appendix \ref{crispr} and Figure \ref{fig:S_f_exp}.

\newpage

\bibliography{ref.bib}
\bibliographystyle{plainnat}

\newpage
\appendix

\section{Proof of Theorem \ref{main_thm}}
\label{proof_main}

\subsection{Notations}
We address the following notations that frequently occur throughout the proof section.
Denote by $f$ an arbitrary linear fitness function $f(x):= \langle \theta, x \rangle$ parameterized by some $\theta \in \mathbb{R}^{d}$ and denote by $F^{\star}$ its maximum. Define $F(S):=\operatorname{avg}_{x\in S} f(x)$, the average fitness under $f$ of population $S$. While $f$ represents arbitrary fitness function, $\{f_{\tilde \theta_t}(x): = \langle \tilde \theta_t, x \rangle\}_{t \in [T]}$ are the linear function parameterized by $\{\tilde \theta_t\}_{t \in [T]}$ obtained by posterior sampling in each iteration of Alg.\ref{alg:one-level(PS)}. Corresponding to each $f_{\tilde \theta_t}$, $F_t^{\star}:= f_{\tilde \theta_t}(x_t^{\star})$ is its maximum value and $x_t^{\star}$ is its one maximum point. Denote by $F_t(S)$ the average $f_{\tilde \theta_t}$ value over $S$. For a clear display, denote by $L$, an upper bound for the $l_2$ norm of any $x_{t,i}$ evaluated, i.e. $\|x_{t,i}\| \leq L$ and in our setting, take $L = \sqrt{d}$. Without clarification $\|\cdot\|$ denotes the $l_2$ norm by default and $\|\cdot\|_{A}$ denotes the norm normalized by matrix $A$.


\subsection{Routine of Alg.\ref{alg:one-level(PS)} and filtrations}

\begin{figure}[h]
    \centering
    \includegraphics[scale=0.6]{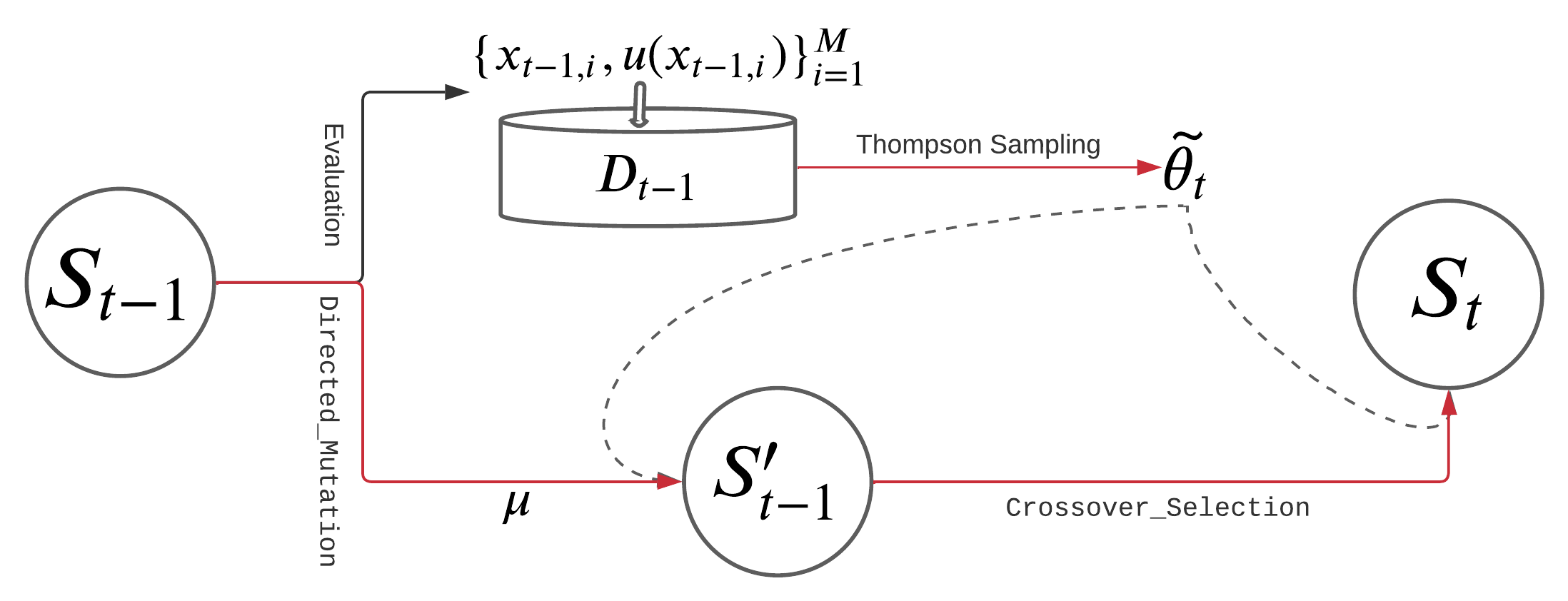}
    \caption{\textbf{Routine of Alg.\ref{alg:one-level(PS)}.} Red lines represent stochastic steps. Dash lines indicate $\tilde \theta_t$ is involved in those steps.}
    \label{fig:stochasticity}
\end{figure}

In Alg.\ref{alg:one-level(PS)}, there are three steps introducing stochasticity. Two of which are calling Module \ref{mod:dire_mutation} as $S_{t-1}^{\prime} = \texttt{Directed\_Mutation}(f_{\tilde \theta_t}, S_{t-1}, \mu)$ and calling  as $S_t = \texttt{Crossover\_Selection}(f_{\tilde \theta_t}, S_{t-1}^{\prime})$. Another one is Thompson sampling step s.t. $\tilde\theta_t$ is sampled from the posterior of $\theta^{\star}$ given data $D_{t-1}$. Fig. \ref{fig:stochasticity} illustrates how these three steps are built into the algorithm routine. 

There are two other sources of stochasticity inherited from the problem setting: the prior of $\theta^{\star}$ (Assumption \ref{asmp_linearf}) and the noisy feedback $\{u(x_{t,i})\}_{i=1}^{M}$ (Assumption \ref{asmp_noise}), which are revealed in the evaluation step. Including all stochasticity, the trajectory of Alg.\ref{alg:one-level(PS)} is
\begin{equation}
\label{equ:trajectory}
    \theta^{\star}, \tilde\theta_1, S_0^{\prime}, S_1, \{u(x_{1,i})\}_{i=1}^{M}, \cdots,  \tilde\theta_{t+1}, S_{t}^{\prime}, S_{t+1}, \{u(x_{t+1,i})\}_{i=1}^{M}, \cdots, \tilde\theta_{T}, S_{T-1}^{\prime}, S_{T}, \{u(x_{T,i})\}_{i=1}^{M}.
\end{equation}

At the convenience of analysis, we introduce multiple lines of the history up to time step $t$ by carefully partitioning the trajectory (\ref{equ:trajectory}), using $\sigma(\cdot)$ to represent the minimal sigma algebra expanded by $\cdot$.

\begin{definition}
Define a filtration $\left\{ \mathcal{H}^{M}_{t} \right\}_{t=0}^{T-1}$ with $\mathcal{H}^{M}_{t}$ be the information accumulated after $t$ rounds of Alg.\ref{alg:one-level(PS)} but before the Directed Mutation step in round $t+1$.
\begin{align*}
    \mathcal{H}^{M}_{0}: &= \sigma \left( \theta^{\star}, \tilde\theta_1 \right),\\
    \mathcal{H}^{M}_{t}: &= \left( \mathcal{H}^{M}_{t-1}, \sigma \left( S_{t-1}^{\prime}, S_t, \{u(x_{t,i})\}_{i=1}^{M}, \tilde\theta_{t+1} \right) \right), \quad t \in [T-1].
\end{align*}
\end{definition}

\begin{definition}
Define a filtration $\left\{ \mathcal{H}^{R}_{t} \right\}_{t=0}^{T-1}$ with $\mathcal{H}^{R}_{t}$ be the information accumulated after $t$ rounds of Alg.\ref{alg:one-level(PS)} but before the Recombination and Selection step in round $t+1$.
\begin{align*}
    \mathcal{H}^{R}_{0}: &= \sigma \left( \theta^{\star}, \tilde\theta_1, S_0^{\prime} \right),\\
    \mathcal{H}^{R}_{t}: &= \left( \mathcal{H}^{R}_{t-1}, \sigma \left(  S_t, \{u(x_{t,i})\}_{i=1}^{M}, \tilde\theta_{t+1}, S_{t}^{\prime} \right) \right), \quad t \in [T-1].
\end{align*}
\end{definition}

\subsection{Property of $\texttt{Directed\_Mutation}$ (Module \ref{mod:dire_mutation})}

Given a fitness function $f(x):= \langle \theta, x\rangle$, a useful observation is that for the dimension where $\theta_i \geq 0$, feature value $1$ is more favorable than $0$ in terms of a higher fitness. So in a population $S$, for each dimension $i$, the ratio of individuals who are with the favored feature is a key quantity, and we define it formally as follows.
\begin{definition}[Ratio of the Favored Feature]
\label{def:ratio_binary}
Under a fitness function $f(x):= \langle \theta, x\rangle$, for a population $S$, define
\begin{equation}
\label{equ:ratio_binary}
    p^{\theta}_i(S) = \left\{ 
    \begin{array}{cc}
        \frac{\#\{x \in S: x_i=0 \}}{|S|} & \theta_i < 0 \\
        \frac{\#\{x \in S: x_i=1 \}}{|S|} & \theta_i \geq 0,
    \end{array}
    \right. \quad \forall i \in [d],
\end{equation}
and we are allowed to omit the superscript $\theta$ of $p^{\theta}_i(S)$ when $\theta$ is clear from the context.
\end{definition}

We show the following property of $\texttt{Directed\_Mutation}$.
\begin{lemma}
\label{lmm:DM}
Suppose $S^{\prime} = \texttt{Directed\_Mutation}(f, S, \mu)$, then the population-averaged fitness of $S^{\prime}$ will not decrease compared to that of $S$ in expectation, that is,
\begin{equation}
\label{equ:dm_utility_exp}
     \mathbb{E} \left[ F(S^{\prime}) \right] \geq F(S).
\end{equation}

And for $\forall \delta \in (0,1)$, if $|S| = \Omega \left(  \frac{\log(\frac{d}{\delta})}{\mu^2} \right)$, then with probability $1-\delta$, 
\begin{equation}
\label{equ:dm_feature_lb}
    p_i\left( S^{\prime} \right) \geq \frac{\mu}{4}, \quad \forall i \in [d].
\end{equation}
\end{lemma}

\begin{proof}
See Appendix \ref{prf_lmm:DM}.
\end{proof}

\subsubsection{High probability events on $\min_i p_i^{\tilde \theta_{t+1}} \left( S^{\prime}_t \right)$}

The Directed Mutation step of Alg.\ref{alg:one-level(PS)} ensures $S^{\prime}_t$ is always sufficient with the feature favored by current $f_{\tilde \theta_{t+1}}$ in each dimension $i$ throughout $T$ rounds, i.e. $\min_i p_i^{\tilde \theta_{t+1}} \left( S^{\prime}_t \right)$ is lower bounded for $\forall t+1 \in [T]$, recall Definition \ref{def:ratio_binary} for the definition of $p^{\theta}_i(S)$.

We introduce the following line of events where this sufficiency of $S^{\prime}_t$ holds and show the intersection of them happens with high probability when the population size $M$ is sufficiently large.

\begin{definition}
\label{def:DM_hpevent}
Define $E^{\operatorname{DM}}_t$ to be the event where $\min_i p_i^{\tilde \theta_{s+1}} \left( S^{\prime}_s \right)$ is lower bounded by $\frac{\mu}{4}$ for $\forall s \leq t$, that is 
\begin{equation}
    E^{\operatorname{DM}}_t:= \{ \forall s \leq t, \quad \min_i p_i^{\tilde \theta_{s+1}}  \left( S^{\prime}_s \right) \geq \frac{\mu}{4}\}.
\end{equation}
Also define $E_{\operatorname{DM}}$ as the intersection of $\{E^{\operatorname{DM}}_t\}_{t=0}^{T-1}$.
\end{definition}

\begin{lemma}
\label{lmm:lb_min_pi}
For $\forall \delta \in (0,1)$, if the population size is sufficiently large s.t. $M = O \left(  \frac{\log(\frac{dT}{\delta})}{\mu^2} \right)$, then
\begin{equation}
    \mathbb{P}\left( E_{\operatorname{DM}} \right) \geq  1-\delta.
\end{equation}
Since event $E_{\operatorname{DM}}$ is independent from the realization of $\theta^{\star}$, thus it still holds with high probability when conditioned on $\theta^{\star}$. Denote by $E^{\theta^{\star}}_{\operatorname{DM}}$, the event $E_{\operatorname{DM}}$ conditioned on $\theta^{\star}$, then
\begin{equation*}
     \mathbb{P}\left( E^{\theta^{\star}}_{\operatorname{DM}} \right) \geq  1-\delta.
\end{equation*}
\end{lemma}

\begin{proof}
See Appendix \ref{prf_lmm:lb_min_pi}.
\end{proof}

\subsection{Linear convergence of \texttt{Crossover\_Selection} (Module \ref{mod:rcb&slc})}

Continuing from Theorem \ref{thm:ascent}, when $S$ is sufficient with the feature favored by $f$ in every dimension, i.e. $\min_i p_i(S)$ is lower bounded, then in expectation, $F(S^{\prime})$ converges linearly to $F^{\star}$ with a nontrivial convergence rate.

\begin{lemma}[Linear Convergence]
\label{lmm:LC_under_conditions}
Suppose $S^{\prime} = \texttt{Crossover\_Selection}(f,S)$, then
    \begin{equation}
        \mathbb{E} \left[ F(S^{\prime}) \right] \geq F(S) + \eta \left( F^{\star} - F(S) \right),
    \end{equation}
with factor $\eta = \frac{\min_i p_i(S)}{\sqrt{2d}}$.
\end{lemma}

\begin{proof}
See Appendix \ref{prf_lmm:LC_under_conditions}.
\end{proof}

\subsection{Thompson Sampling}
According to Assumption \ref{asmp_noise}, in the dataset $D_t = \left\{ \left\{x_{s,i}, u(x_{s,i}) \right\}_{i = 1}^{M}, s= [t]  \right\}$
\begin{equation}
\label{equ:noise}
    u(x_{s,i}) = f_{\theta^{\star}}(x_{s,i}) + \xi_{s,i},
\end{equation}
with $\xi_{s,i}$ i.i.d. sampled from $\mathcal{N}( 0, \sigma^2)$ and independent from all other stochasticity. 

Therefore, by Bayes' Rule, the posterior of $\theta^{\star}$ give $D_{t-1}$ is also Gaussian distributed, for $\forall t \in [T]$:
\begin{equation}
    \label{equ:TS_dist} \tilde{\theta}_{t} \sim \mathcal{N}(\hat \theta_{t}, V^{-1}_{t}),
\end{equation}
where
\begin{align}
    \label{equ:cov_matrix} V_{t} &= \frac{1}{\sigma^2} \Phi^{\top}_{t-1} \Phi_{t-1} + \lambda I,\\
    \label{equ:hat_theta} \hat \theta_{t} &= \frac{1}{\sigma^2} V^{-1}_{t} \Phi^{\top}_{t-1} U_{t-1}, 
\end{align}
recall from Alg. \ref{alg:one-level(PS)} for the updating rules of $\Phi_t$ and $U_t$.

Given the posterior distribution (\ref{equ:TS_dist}), we are able to show $\tilde{\theta}_{t}$ concentrates to $\theta^{\star}$ in term of the normalized distance between them.

\subsubsection{High probability events on $\left\|\tilde{\theta}_{t}- \theta^{\star} \right\|_{V_{t}}$}

We introduce two useful lines of high probability events similar to those in \citet{abeille2017linear}, except here these events are defined conditioned on any realization of $\theta^{\star}$ sampled from its prior. We rephrased the definition to match our notations.
\begin{definition}
Conditioned on $\theta^{\star}$, for any given probability tolerance $\delta \in (0,1)$, each time step $t \in [T]$ and a line of ($\theta^{\star}$ dependent) radiuses $\{\beta^{\theta^{\star}}_{t}\left(\delta \right)\}_{t=1}^{T}$, we define $\hat{E}_{t}^{\theta^{\star}}$ as the event where $\hat \theta_{s}$ concentrates around $\theta^{\star}$ for all previous steps $s \leq t$, i.e. 
\begin{equation}
    \hat{E}^{\theta^{\star}}_{t}(\delta)=\left\{ \left. \forall s \leq t,\left\|\hat \theta_{s}-\theta^{\star}\right\|_{V_{s}} \leq \beta^{\theta^{\star}}_{s}\left(\delta \right) \right \vert \theta^{\star} \right\}. 
\end{equation}
with a line of ($\theta^{\star}$ independent) radiuses $\{\alpha_{t}\left(\delta \right)\}_{t=1}^{T}$, we also define $\tilde{E}_{t}^{\theta^{\star}}$ as the event where the sampled parameter $\tilde{\theta}_{s}$ concentrates around $\hat \theta_{s}$ for all steps $s \leq t$, i.e.
\begin{equation}
    \tilde{E}^{\theta^{\star}}_{t}(\delta)=\left\{ \left. \forall s \leq t,\left\|\tilde{\theta}_{s}-\hat \theta_{s}\right\|_{V_{s}} \leq \alpha_{s}\left(\delta\right) \right \vert \theta^{\star} \right\}.
\end{equation}
Then under the same $\delta$ and $\theta^{\star}$, which are omitted here, we have $\hat{E}:=\hat{E}_{T} \subset \cdots \subset \hat{E}_{1}$,
$\tilde{E}:=\tilde{E}_{T} \subset \cdots \subset \tilde{E}_{1}$ and define $E^{\theta^{\star}}(\delta) := \hat{E}^{\theta^{\star}}(\delta) \cap \tilde{E}^{\theta^{\star}}(\delta)$.
\end{definition}

With appropriate choices of $\{\beta_{t}\}$ and $\{\alpha_{t}\}$, event $E^{\theta^{\star}}(\delta)$ defined above happens with high probability as stated in the following lemma.

\begin{lemma}
\label{lmm:hp_events}
Under Assumption \ref{asmp_linearf} and \ref{asmp_noise}, conditioned on any realization of $\theta^{\star}$ drawn from its prior, for $\forall \delta \in (0,1)$ and any series of feature vectors $\left( \{x_{1,i}\}_{i = 1}^M , \cdots, \{x_{T,i}\}_{i = 1}^M \right)$ where each  $\|x_{t,i}\| \leq L$,  $\mathbb{P}\left( E^{\theta^{\star}} \left( \frac{\delta}{2} \right) \right) \geq  1-\delta$ with $\beta_{t}^{\theta^{\star}} \left(\delta \right)$ and $\alpha_{t}\left(\delta \right)$ specified as
\begin{align}
    \label{beta_t} \beta_{t}^{\theta^{\star}} \left(\delta \right) &= \sqrt{2 \log \left( \frac{1}{\delta} \right) + d \log\left(\frac{\sigma^2 \lambda d + t M L^2}{ \sigma^2 \lambda d}\right)} + \sqrt{\lambda} \|\theta^{\star}\|, \quad \forall t \in [T].\\
    \label{alpha_t}\alpha_{t}\left(\delta \right) &= 2\sqrt{d \log \left( \frac{T}{\delta} \right) } + \sqrt{d} , \quad \forall t \in [T].
\end{align}
\end{lemma}

\begin{proof}
See Appendix \ref{prf_lmm:hp_events}.
\end{proof}

\subsection{Prediction error under batch update}
Before regret decomposition, one more preparation to have is a modified concentration on the accumulated prediction error of $\tilde\theta_t$ catering for the batch-data update routine in Alg.\ref{alg:one-level(PS)}. In the following lemma, we summarize a more general version of this concentration result.
\begin{lemma}
\label{lmm:batch_update}
Suppose at any timestep $a_t$ lies in a confidence ellipsoid around $b_t$ in the sense that 
\begin{equation*}
    \| a_t - b_t \|_{V_{t}} \leq \eta_{t}(\delta), \quad \forall t \in [T],
\end{equation*}
and $\|x_{t,i}\| \leq L , \forall t \in [T], i \in [M]$, then it holds that,
\begin{align}
\nonumber & \sum_{t=1}^T \sum_{i=1}^M \left| \langle a_t - b_t, x_{t,i} \rangle \right|\\
\label{def:pred_err}\leq & \eta_T(\delta) \sqrt{ \frac{2L^2 + 2\lambda}{\lambda}} \cdot \sqrt{ dMT \log\left(\frac{\sigma^2 d \lambda + M T L^2}{\sigma^2 d \lambda}\right)}  + \eta_T(\delta) \frac{ 2L}{\sqrt{\lambda}} \cdot dM \log\left(\frac{\sigma^2 d \lambda + M T L^2}{\sigma^2 d \lambda}\right).
\end{align}

And let us give an alias $\operatorname{RGT} \left( \eta_T(\delta) \right)$ to the RHS of (\ref{def:pred_err}).
\end{lemma}

\begin{proof}
See Appendix \ref{prf_lmm:batch_update}.
\end{proof}

\subsection{Regret decomposition}

Recall the notation that $x^{\star}$ is a maximum point of $f_{\theta^{\star}}$.

By the scheme of posterior sampling, $f_{\theta^{\star}}(x^{\star})$ and $f_{\tilde \theta_t}(x_t^{\star})$ are identically distributed conditioned on $D_{t-1}$, which leads to
\begin{equation}
\label{equ:PS}
    \mathbb{E}\left[ \left. f_{\theta^{\star}}(x^{\star}) - f_{\tilde \theta_t}(x_t^{\star}) \right \vert D_{t-1} \right] = 0.
\end{equation}

With expectation taken over all 
stochasticity, the per-round Bayesian regret is
\begin{align*}
        \mathbb{E} \left[ \sum_{i=1}^M \left( f_{\theta^{\star}}(x^{\star}) - f_{\theta^{\star}}(x_{t,i}) \right) \right]
        &= \sum_{i=1}^M \mathbb{E} \left[ f_{\theta^{\star}}(x^{\star}) - f_{\tilde \theta_t}(x_t^{\star}) \right] + \sum_{i=1}^M \mathbb{E} \left[ f_{\tilde \theta_t}(x_t^{\star}) - f_{\theta^{\star}}(x_{t,i}) \right]\\ 
        &= \sum_{i=1}^M \mathbb{E} \left[ \mathbb{E}\left[ \left. f_{\theta^{\star}}(x^{\star}) - f_{\tilde \theta_t}(x_t^{\star}) \right \vert D_{t-1} \right] \right] + \sum_{i=1}^M \mathbb{E} \left[ f_{\tilde \theta_t}(x_t^{\star}) - f_{\theta^{\star}}(x_{t,i}) \right]\\
        &\overset{(\ref{equ:PS})}{=} \sum_{i=1}^M \mathbb{E} \left[ f_{\tilde \theta_t}(x_t^{\star}) - f_{\theta^{\star}}(x_{t,i}) \right]
        \\
        &= \sum_{i=1}^M \mathbb{E} \left[ f_{\tilde \theta_t}(x_t^{\star}) - f_{\tilde \theta_t}(x_{t,i}) + f_{\tilde \theta_t}(x_{t,i}) - f_{\theta^{\star}}(x_{t,i}) \right]\\
        &= \mathbb{E} \left[M \left( F^{\star}_t - F_t(S_t) \right) \right] + \mathbb{E} \left[\sum_{i=1}^M \langle \tilde\theta_t -\theta^{\star}, x_{t,i} \rangle \right].
\end{align*}
Then the total Bayesian regret over $T$ rounds sums up to be 
\begin{align}
    \nonumber \operatorname{BayesRGT}(T,M) &= \mathbb{E} \left[M \sum_{t=1}^T  F^{\star}_t - F_t(S_t) \right] + \mathbb{E} \left[ \sum_{t=1}^T \sum_{i=1}^M \langle \tilde\theta_t -\theta^{\star}, x_{t,i} \rangle \right]\\
    \label{rgt_dcp} &= \mathbb{E}_{\theta^{\star} \sim \pi}\left[ \mathbb{E}_{\theta^{\star}} \left[ M \sum_{t=1}^T  F^{\star}_t - F_t(S_t) \right] \right] + \mathbb{E}_{\theta^{\star} \sim \pi} \left[ \mathbb{E}_{\theta^{\star}} \left[  \sum_{t=1}^T \sum_{i=1}^M \langle \tilde\theta_t -\theta^{\star}, x_{t,i} \rangle \right] \right],
\end{align}
where $\mathbb{E}_{\theta^{\star}} \left[ \cdot \right]$ denotes the conditional expectation on a given $\theta^{\star}$: $\mathbb{E}\left[ \left. \cdot \right \vert \theta^{\star} \right]$.

Note that under any realization of $\theta^{\star}$, the regret of each individual at any time step should be no more than the range of $f_{\theta^{\star}}$ on domain $\mathcal{X}$. For any $\theta \in R^{d}$ parameterizing the fitness $f_{\theta}$ as $f_{\theta}(x) = \langle \theta, x\rangle$, denote by $B_f^{\theta}$ an upper bound for the range of $f_{\theta}$, i.e.
\begin{equation}
    B_f^{\theta} := 2 L \|\theta\| \geq \max_{x} f_{\theta}(x) - \min_{x} f_{\theta}(x).
\end{equation}
For the regret of each individual in each step, it holds that
\begin{equation}
   f_{\theta^{\star}}(x^{\star}) - f_{\theta^{\star}}(x_{t,i}) \leq 2 \|\theta^{\star}\| L = B_f^{\theta^{\star}}.
\end{equation}

Therefore, when bounding the total regret decomposed as (\ref{rgt_dcp}), it is reasonable to  truncate terms in the RHS of (\ref{rgt_dcp}) with $B_f^{\theta^{\star}}$ to derive a tighter bound.
\begin{align}
    \operatorname{BayesRGT}(T,M)
    \label{equ:rgt_p1} \leq &  \mathbb{E}_{\theta^{\star} \sim \pi} \left[ M \mathbb{E}_{\theta^{\star}} \left[  \sum_{t=1}^T \min \left\{ F^{\star}_t - F_t(S_t), B_f^{\theta^{\star}} \right\}  \right] \right]  \\
    \label{equ:rgt_p2} & + \mathbb{E}_{\theta^{\star} \sim \pi} \left[ \mathbb{E}_{\theta^{\star}} \left[  \sum_{t=1}^T \sum_{i=1}^M \min \left\{ \left| \langle \tilde\theta_t -\theta^{\star}, x_{t,i} \rangle \right|, B_f^{\theta^{\star}} \right\} \right] \right]
\end{align}

\subsection{Bounding the first half (\ref{equ:rgt_p1})}
\subsubsection{After calling  $S_{t-1}^{\prime} = \texttt{Directed\_Mutation}(f_{\tilde \theta_t}, S_{t-1}, \mu)$}

As shown in Lemma \ref{lmm:DM}, the population average of $S_{t}^{\prime}$ under $f_{\tilde \theta_{t+1}}$ in not decreasing from that of $S_{t}$, that is 
\begin{equation}
\label{equ:f_S'}
    \mathbb{E} \left[\left. F_{t+1}(S_t^{\prime}) \right \vert  \mathcal{H}^{M}_{t} \right] \geq F_{t+1}(S_t).
\end{equation}

The other property of $\texttt{Directed\_Mutation}$ is to ensure that w.h.p. $\min_i p_i^{\tilde \theta_{t+1}} \left( S^{\prime}_t \right)$ is lower bounded for $\forall t+1 \in [T]$, which is stated in the definition of event $E^{\theta^{\star}}_{\operatorname{DM}}$ (Definition \ref{def:DM_hpevent}). 
So from here on, given any realization of $\theta^{\star}$, our further analysis is conditioned on $E^{\theta^{\star}}:= E^{\theta^{\star}} \left( \frac{\delta}{2} \right) \cap E^{\theta^{\star}}_{\operatorname{DM}}$.

\begin{corollary}
\label{cor:hp_2conditons}
Given any realization of $\theta^{\star}$, if $M = \Omega \left(  \frac{\log(\frac{dT}{\delta})}{\mu_M^2} \right)$, then $\mathbb{P} \left( E^{\theta^{\star}} \right) \geq 1- 2\delta$ for $\forall \delta \in (0,1)$. Conditioned on $E^{\theta^{\star}}:= E^{\theta^{\star}} \left( \frac{\delta}{2} \right) \cap E^{\theta^{\star}}_{\operatorname{DM}}$, it is guaranteed that   
\begin{align}
    \label{equ:cond_pi}&\min_i p_i^{\tilde \theta_{t+1}}(S^{\prime}_t) \geq C_{S^{\prime}}:= \frac{\mu}{4}, \quad \forall t+1 \in [T],\\
    \label{equ:theta_acc}&\| \tilde \theta_{t} - \theta^{\star} \|_{V_{t}} \leq \beta^{\theta^{\star}}_t\left(\frac{\delta}{2}\right) + \alpha_t\left(\frac{\delta}{2}\right), \quad \forall t \in [T].
\end{align}
where recall the definition of $\beta^{\theta^{\star}}_t$ and $\alpha_t$ from (\ref{beta_t}) and (\ref{alpha_t}).
\end{corollary}

\begin{proof}
The proof is directly derived by combining Lemma \ref{lmm:lb_min_pi} and Lemma \ref{lmm:hp_events}.
\end{proof}

\subsubsection{After calling  $S_t = \texttt{Crossover\_Selection}(f_{\tilde \theta_t}, S_{t-1}^{\prime})$}

Conditioned on $\theta^{\star}$, we are about to decompose $\sum_{t=1}^T \left( F^{\star}_t - F_t(S_t) \right)$ by leveraging the linear convergence property shown in Lemma \ref{lmm:LC_under_conditions}. Conditionally on $E^{\theta^{\star}}$, applying Lemma \ref{lmm:LC_under_conditions} to each call of $\texttt{Crossover\_Selection}(f_{\tilde \theta_t}, S_{t-1}^{\prime})$ guarantees for $\forall t+1 \in [T],$
\begin{align}
    \nonumber\mathbb{E}_{E^{\theta^{\star}}} \left[ \left. F_{t+1}(S_{t+1})  \right \vert  \mathcal{H}^{R}_{t} \right] &\geq \mathbb{E}_{E^{\theta^{\star}}} \left[ \left. F_{t+1}(S_{t}^{\prime}) +  \frac{ \min_i p_i^{\tilde \theta_{t+1}}(S^{\prime}_t)}{\sqrt{2d}} \left( F^{\star}_{t+1} - F_{t+1}(S_{t}^{\prime}) \right) \right \vert  \mathcal{H}^{R}_{t} \right] \\
    \label{equ:lc} &\overset{(\ref{equ:cond_pi})}{\geq} F_{t+1}(S_{t}^{\prime}) +  \frac{ C_{S^{\prime}}}{\sqrt{2d}} \left( F^{\star}_{t+1} - F_{t+1}(S_{t}^{\prime}) \right),
\end{align}
where $\mathbb{E}_{E^{\theta^{\star}}} \left[ \cdot \right]$ is the conditional expectation on event $E^{\theta^{\star}}$.

Recall (\ref{equ:f_S'}) that
\begin{equation}
\tag{\ref{equ:f_S'} revisited}
    \mathbb{E} \left[\left. F_{t+1}(S_t^{\prime}) \right \vert  \mathcal{H}^{M}_{t} \right] \geq F_{t+1}(S_t).
\end{equation}

Conditioned on $E^{\theta^{\star}}$, it still holds that
\begin{equation}
\label{equ:dm_non_dec_cond}
    \mathbb{E}_{E^{\theta^{\star}}} \left[\left. F_{t+1}(S_t^{\prime}) \right \vert  \mathcal{H}^{M}_{t} \right] \geq F_{t+1}(S_t),
\end{equation}
since in $E^{\theta^{\star}}$, $E^{\theta^{\star}} \left( \frac{\delta}{2} \right)$ holds independent from the Directed Mutation step $S_{t}^{\prime} = \operatorname{DM}(f_{\tilde \theta_{t+1}}, S_{t}, \mu_M)$, and conditioned on $E^{\theta^{\star}}_{\operatorname{DM}}$, $\left. F_{t+1}(S_t^{\prime}) \right \vert  \mathcal{H}^{M}_{t}$ tends to be greater then it was unconditionally. 

Along with $\mathcal{H}^{M}_{t} \subset \mathcal{H}^{R}_{t}$, we have
\begin{align}
    \mathbb{E}_{E^{\theta^{\star}}} \left[ \left. F_{t+1}(S_{t+1})  \right \vert  \mathcal{H}^{M}_{t} \right] &\geq \mathbb{E}_{E^{\theta^{\star}}} \left[ \left. \mathbb{E}_{E^{\theta^{\star}}} \left[ \left. F_{t+1}(S_{t+1})  \right \vert  \mathcal{H}^{R}_{t} \right] \right \vert  \mathcal{H}^{M}_{t} \right] \\
    &\overset{(\ref{equ:lc})}{\geq} \mathbb{E}_{E^{\theta^{\star}}} \left[ \left. F_{t+1}(S_{t}^{\prime}) + \frac{ C_{S^{\prime}}}{\sqrt{2d}} \left( F^{\star}_{t+1} - F_{t+1}(S_{t}^{\prime}) \right) \right \vert  \mathcal{H}^{M}_{t} \right]\\
    &\overset{(\ref{equ:dm_non_dec_cond})}{\geq} F_{t+1}(S_{t}) + \frac{ C_{S^{\prime}}}{\sqrt{2d}} \left( F^{\star}_{t+1} - F_{t+1}(S_{t}) \right).
\end{align}

By introducing the convergence rate $\gamma:= 1- \frac{ C_{S^{\prime}}}{\sqrt{2d}}$ s.t. $\frac{1}{1-\gamma} =O \left( \frac{\sqrt{d}}{\mu} \right) $ and an residual term 
\begin{equation*}
    e_{t+1}:=  \mathbb{E}_{E^{\theta^{\star}}} \left[ \left. F_{t+1}(S_{t+1})  \right \vert  \mathcal{H}^{M}_{t} \right] - F_{t+1}(S_{t+1}),
\end{equation*}
we have
\begin{equation}
    F^{\star}_{t+1} - F_{t+1}(S_{t+1}) \leq \gamma(F^{\star}_{t+1} - F_{t+1}(S_{t})) + e_{t+1},
\end{equation}
where $\{e_t\}_{t=1}^{T}$ is a martingale difference with 
\begin{equation}
    \label{equ:et_mean0} \mathbb{E}_{E^{\theta^{\star}}}[e_{t+1} \mid \mathcal{H}^{M}_{t}] = 0.
\end{equation}

Thus,
\begin{align*}
    F^{\star}_{t+1} - F_{t+1}(S_{t+1}) &\leq \gamma(F^{\star}_{t+1} - F_{t+1}(S_{t})) + e_{t+1}\\
    &= \gamma \left[ F^{\star}_{t} - F_{t}(S_{t}) + F^{\star}_{t+1}- F^{\star}_{t} + F_{t}(S_{t}) - F_{t+1}(S_{t})\right] + e_{t+1}\\
    &= \gamma \left[F^{\star}_{t} - F_{t}(S_{t}) \right] + \gamma \left[ F^{\star}_{t+1}- F^{\star}_{t} + F_{t}(S_{t}) - F_{t+1}(S_{t})\right] + e_{t+1}.
\end{align*}

Therefore we have the recursion that
\begin{equation}
    F^{\star}_{t} - F_{t}(S_{t}) \leq 
    \gamma^{t} \left(F^{\star}_{1} - F_{1}(S_{0}) \right) + \sum^{t-1}_{k=1}\gamma^{t-k} \left( F^{\star}_{k+1}- F^{\star}_{k} \right) + \sum^{t-1}_{k=1}\gamma^{t-k} \left( F_{k}(S_{k}) - F_{k+1}(S_{k}) \right) + \sum^{t}_{k=1}\gamma^{t-k} e_k,
\end{equation}

summing up which from $t=1$ to $T$ gives
\begin{align}
    \label{equ:term1} \sum_{t=1}^T F^{\star}_t - F_t(S_t) \leq&  \sum_{t = 1}^T \sum^{t}_{k=1}\gamma^{t-k} e_k \\
    \label{equ:term2} &+ \sum_{t = 1}^T  \gamma^{t} \left(F^{\star}_{1} - F_{1}(S_{0}) \right) \\ 
    \label{equ:term3} &+  \sum_{t = 1}^T \sum^{t-1}_{k=1}\gamma^{t-k} \left( F^{\star}_{k+1}- F^{\star}_{k} \right) \\ 
    \label{equ:term4} &+ \sum_{t = 1}^T \sum^{t-1}_{k=1}\gamma^{t-k} \left( F_{k}(S_{k}) - F_{k+1}(S_{k})\right).
\end{align}

As it appears in (\ref{equ:rgt_p1}), what matters in bounding regret is the expected truncated value of $\sum_{t = 1}^T F^{\star}_{t} - F_{t}(S_{t})$, which is 
\begin{equation*}
    \mathbb{E}_{\theta^{\star} \sim \pi} \left[ M \mathbb{E}_{\theta^{\star}} \left[  \sum_{t=1}^T \min \left\{ F^{\star}_t - F_t(S_t), B_f^{\theta^{\star}} \right\}  \right]  \right] ,\tag{\ref{equ:rgt_p1} revisited}
\end{equation*}

and the decomposition of $\sum_{t = 1}^T F^{\star}_{t} - F_{t}(S_{t})$ into four terms as above holds conditionally on $E^{\theta^{\star}}$. So from here on, we progress with first upper bounding $\mathbb{E}_{\theta^{\star}} \left[  \sum_{t=1}^T \min \left\{ F^{\star}_t - F_t(S_t), B_f^{\theta^{\star}} \right\}  \right]$ by
\begin{equation}
\label{equ:P1_dcp_cond_theta*}
    \mathbb{E}_{\theta^{\star}} \left[  \sum_{t=1}^T \min \left\{ F^{\star}_t - F_t(S_t), B_f^{\theta^{\star}} \right\}  \right]  \leq 2 \delta T B_f^{\theta^{\star}} + (1- 2\delta) \mathbb{E}_{E^{\theta^{\star}}} \left[   \sum_{t=1}^T  F^{\star}_t - F_t(S_t) \right],
\end{equation}
and then upper bounding $\mathbb{E}_{E^{\theta^{\star}}} \left[   \sum_{t=1}^T  F^{\star}_t - F_t(S_t) \right]$ with

\begin{align}
    \mathbb{E}_{E^{\theta^{\star}}} \left[   \sum_{t=1}^T  F^{\star}_t - F_t(S_t) \right] 
    \label{equ:term1} \leq &  \mathbb{E}_{E^{\theta^{\star}}} \left[ \sum_{t = 1}^T \sum^{t}_{k=1}\gamma^{t-k} e_k \right] \\
    \label{equ:term2} &+ \mathbb{E}_{E^{\theta^{\star}}} \left[  \sum_{t = 1}^T  \gamma^{t} \left(F^{\star}_{1} - F_{1}(S_{0}) \right) \right]   \\ 
    \label{equ:term3} &+  \mathbb{E}_{E^{\theta^{\star}}} \left[ \sum_{t = 1}^T \sum^{t-1}_{k=1}\gamma^{t-k} \left( F^{\star}_{k+1}- F^{\star}_{k} \right) \right] \\ 
    \label{equ:term4} &+ \mathbb{E}_{E^{\theta^{\star}}} \left[ \sum_{t = 1}^T \sum^{t-1}_{k=1}\gamma^{t-k} \left( F_{k}(S_{k}) - F_{k+1}(S_{k})\right)\right] .
\end{align}

\subsubsection{Term (\ref{equ:term1})}
$\{e_t\}_{t=1}^{T}$ is claimed to be a martingale difference when first being introduced, that is, recall (\ref{equ:et_mean0}) that
\begin{equation}
\tag{\ref{equ:et_mean0} revisited}
    \mathbb{E}_{E^{\theta^{\star}}}[e_{t+1} \mid \mathcal{H}^{M}_{t}] = 0, \quad \forall t+1 \in [T].
\end{equation}
Thus by the property of martingale difference,
\begin{equation*}
    \mathbb{E}_{E^{\theta^{\star}}} \left[e_{t+1} \right] = \mathbb{E}_{E^{\theta^{\star}}} \left[ \mathbb{E}_{E^{\theta^{\star}}} \left[e_{t+1} \mid \mathcal{H}^{M}_{t} \right]\right]= 0, \quad \forall t+1 \in [T].
\end{equation*}
Then by the linearity of $\mathbb{E}_{E^{\theta^{\star}}} \left[ \cdot \right]$:
\begin{align}
    \nonumber \mathbb{E}_{E^{\theta^{\star}}} \left[ \sum_{t = 1}^T \sum^{t}_{k=1}\gamma^{t-k} e_k \right] &= \sum_{t = 1}^T \sum^{t}_{k=1} \mathbb{E}_{E^{\theta^{\star}}} \left[ \gamma^{t-k} e_k \right]\\
    \nonumber &=  \sum_{t = 1}^T \sum^{t}_{k=1} \gamma^{t-k} \mathbb{E}_{E^{\theta^{\star}}} \left[  e_k \right]\\
     \label{equ:b_term1}&= 0.
\end{align} 

\subsubsection{Term (\ref{equ:term2})}
Before looking into the term (\ref{equ:term2}), we first introduce the following lemma upper bounding the expectation of $\tilde \theta_{t}$'s $l_2$ norm conditioned on event $E^{\theta^{\star}}$.

\begin{lemma}
\label{lmm:E_theta_t_l2}
For $\forall t \in [T]$, $\mathbb{E}_{E^{\theta^{\star}}} \left[ \| \tilde \theta_t \| \right]$ has the following upper bound.
\begin{equation}
   \mathbb{E}_{E^{\theta^{\star}}} \left[ \| \tilde \theta_t \| \right] \leq  2 \| \theta^{\star} \| + 2 \sqrt{\frac{d}{\lambda}}.
\end{equation}
\end{lemma}

\begin{proof}
See Appendix \ref{prf_lmm:E_theta_t_l2}.
\end{proof}

What is to take expectation in (\ref{equ:term2}) is of constant order because
\begin{equation}
\label{equ:t1}
     \sum_{t = 1}^T  \gamma^{t-1} \left(F^{\star}_{1} - F_{1}(S_{0}) \right) \leq   \frac{1}{1-\gamma}  \left| F^{\star}_{1} - F_{1}(S_{0}) \right| \leq  \frac{1}{1-\gamma} B_f^{\tilde\theta_1},
\end{equation}
where $B_f^{\tilde \theta_1} = 2L \| \tilde \theta_1 \|$.

Then by taking expectation over both sides of (\ref{equ:t1}), we have
\begin{equation}
\label{equ:b_term2}
    \mathbb{E}_{E^{\theta^{\star}}} \left[ \sum_{t = 1}^T \gamma^{t} \left(F^{\star}_{1} - F_{1}(S_{0}) \right) \right] \leq  \frac{2 L}{1-\gamma}  \mathbb{E}_{E^{\theta^{\star}}} \left[ \| \tilde \theta_1 \| \right] \leq \frac{2}{1-\gamma} \left( B_f^{\theta^{\star}} + 2 \sqrt{\frac{d }{\lambda}}L\right).
\end{equation}

\subsubsection{Term (\ref{equ:term3})}
Rearrange terms to sum up in (\ref{equ:term3}) as
\begin{align*}
    \sum_{t = 1}^T \sum^{t-1}_{k=1} \gamma^{t-k} \left( F^{\star}_{k+1}- F^{\star}_{k} \right) &= \sum_{k = 1}^{T-1} \left( F^{\star}_{k+1}- F^{\star}_{k} \right)  \sum_{t=k+1}^{T} \gamma^{t-k}\\
    &= \sum_{k = 1}^{T-1} \frac{\gamma - \gamma^{T-k+1}}{1-\gamma}  \left( F^{\star}_{k+1}- F^{\star}_{k} \right) \\
    &= \frac{\gamma}{1-\gamma} ( F^{\star}_{T}- F^{\star}_{1} ) - \sum_{k = 1}^{T-1} \frac{\gamma^{T-k+1}}{1-\gamma}  \left( F^{\star}_{k+1}- F^{\star}_{k} \right)\\
    &= \frac{\gamma}{1-\gamma} ( F^{\star}_{T}- F^{\star}_{1} ) - \frac{\gamma^2}{1-\gamma}F^{\star}_{T} + \sum_{k = 2}^{T-1}  \gamma^{T-k+1} F^{\star}_{k} + \frac{\gamma^T}{1-\gamma}F^{\star}_{1}\\
     &= \sum_{k = 2}^{T}  \gamma^{T-k+1} F^{\star}_{k} -\gamma^{T-1} F^{\star}_{1}.
\end{align*}

Thus by taking expectation conditioned on $E^{\theta^{\star}}$ over the absolute value of RHS, we have
\begin{align}
    \nonumber \mathbb{E}_{E^{\theta^{\star}}} \left[ \sum_{t = 1}^T \sum^{t-1}_{k=1}\gamma^{t-k} \left( F^{\star}_{k+1}- F^{\star}_{k} \right) \right] &\leq \mathbb{E}_{E^{\theta^{\star}}} \left[ \sum_{k = 2}^{T} \gamma^{T-k+1} \left| F^{\star}_{k} \right| \right]+ \mathbb{E}_{E^{\theta^{\star}}} \left[\gamma^{T-1} \left| F^{\star}_{1} \right| \right]\\
    \nonumber &\leq L  \sum_{k = 2}^{T} \gamma^{T-k+1} \cdot \mathbb{E}_{E^{\theta^{\star}}} \left[   \|\tilde\theta_{k}\| \right]+ L \gamma^{T-1} \cdot \mathbb{E}_{E^{\theta^{\star}}} \left[  \|\tilde\theta_{1}\| \right]\\
    \nonumber&\leq  \sum_{k = 0}^{T-1} \gamma^{k} \left( 2 \| \theta^{\star} \| L + 2 \sqrt{\frac{d}{\lambda}} L \right)\\
    \label{equ:b_term3}&\leq \frac{1}{1-\gamma} \left( B_f^{\theta^{\star}}  + 2 \sqrt{\frac{d}{\lambda}} L \right).
\end{align}

\subsubsection{Term (\ref{equ:term4})}
We start off by rearranging terms in the summation: $\sum_{t = 1}^T \sum^{t-1}_{k=1}\gamma^{t-k} \left( F_{k}(S_{k}) - F_{k+1}(S_{k})\right) $.
 
\begin{align}
    \nonumber \sum_{t = 1}^T \sum^{t-1}_{k=1}\gamma^{t-k}  \left( F_{k}(S_{k}) - F_{k+1}(S_{k})\right) &=
    \sum_{k = 1}^{T-1} \left( F_{k}(S_{k}) - F_{k+1}(S_{k})\right)   \sum_{t=k+1}^{T} \gamma^{t-k}\\
    \nonumber &= \sum_{k = 1}^{T-1} \frac{\gamma - \gamma^{T-k+1}}{1-\gamma}  \left( F_{k}(S_{k}) - F_{k+1}(S_{k})\right)\\
    \nonumber &= \sum_{k = 1}^{T-1} \frac{\gamma - \gamma^{T-k+1}}{1-\gamma}  \langle \tilde \theta_{k} -  \tilde \theta_{k+1}, \frac{1}{M} \sum_{x \in S_k} x  \rangle\\
    \label{equ:gap_btw_2stps}&\leq \frac{1}{1-\gamma} \frac{1}{M} \sum_{t = 1}^{T-1} \sum_{i = 1}^{M}  \left|\langle \tilde \theta_{t} -  \tilde \theta_{t+1}, x_{t,i} \rangle \right|.
\end{align}

In the following Corollary \ref{cor:consecutive_diff}, we bound the RHS above by constructing a high probability confidence ellipsoid for $\tilde \theta_{t+1} - \tilde \theta_{t}$ and then completing with a call of Lemma \ref{lmm:batch_update}.

\begin{corollary}
\label{cor:consecutive_diff}
For any realization of $\theta^{\star}$, conditioned on event $E^{\theta^{\star}}$, it holds that 
\begin{equation}
    \sum_{t = 1}^T \sum^{t-1}_{k=1}\gamma^{t-k}  \left( F_{k}(S_{k}) - F_{k+1}(S_{k})\right) \leq \frac{1}{1-\gamma} \frac{1}{M} \operatorname{RGT} \left( 2\beta^{\theta^{\star}}_T\left(\frac{\delta}{2}\right) + 2\alpha_T\left(\frac{\delta}{2}\right) \right) .
\end{equation}
\end{corollary}

\begin{proof}
See Appendix \ref{prf_cor:consecutive_diff}.
\end{proof}

Therefore, after taking expectation conditioned on $E^{\theta^{\star}}$, we still have
\begin{equation}
\label{equ:b_term4}
    \mathbb{E}_{E^{\theta^{\star}}} \left[ \sum_{t = 1}^T \sum^{t-1}_{k=1} \gamma^{t-k}  \left( F_{k}(S_{k}) - F_{k+1}(S_{k})\right) \right] 
    \leq  \frac{1}{M} \cdot \frac{1}{1-\gamma} \operatorname{RGT} \left( 2\beta^{\theta^{\star}}_T\left(\frac{\delta}{2}\right) + 2\alpha_T\left(\frac{\delta}{2}\right)\right).
\end{equation}

\subsubsection{Pulling $4$ terms into the final bound of the first half (\ref{equ:rgt_p1})}

Going back to  the contribution coming from the first half of the regret decomposition (\ref{rgt_dcp}), plugging (\ref{equ:b_term1}), (\ref{equ:b_term2}), (\ref{equ:b_term3}) and (\ref{equ:b_term4}) into (\ref{equ:P1_dcp_cond_theta*}), it holds that, for $\forall \delta \in (0,1)$
\begin{align}
    \nonumber M \mathbb{E}_{\theta^{\star}} \left[  \sum_{t=1}^T \min \left\{ F^{\star}_t - F_t(S_t), B_f^{\theta^{\star}} \right\}  \right]  \leq& 2 \delta MT \cdot B_f^{\theta^{\star}} + (1- 2\delta) M \cdot \mathbb{E}_{E^{\theta^{\star}}} \left[  \sum_{t=1}^T  F^{\star}_t - F_t(S_t) \right]\\
    \nonumber \leq& 2 \delta M T \cdot B_f^{\theta^{\star}} +  (1- 2\delta) \cdot \frac{3M}{1-\gamma} \left( B_f^{\theta^{\star}}  + 2 \sqrt{\frac{d}{\lambda}} L \right).\\
    \label{equ:P1_cond_theta*}&+ (1- 2\delta) \cdot \frac{1}{1-\gamma} \operatorname{RGT} \left( 2\beta^{\theta^{\star}}_T\left(\frac{\delta}{2}\right) + 2\alpha_T\left(\frac{\delta}{2}\right) \right).
\end{align}

Averaging (\ref{equ:P1_cond_theta*}) over the prior of $\theta^{\star}$, we have
\begin{align}
    \nonumber M  \mathbb{E} \left[\sum_{t=1}^T \min \left\{ F^{\star}_t - F_t(S_t), B_f^{\theta^{\star}} \right\}\right] =&   \mathbb{E}_{\theta^{\star} \sim \pi} \left[ M \mathbb{E}_{\theta^{\star}} \left[  \sum_{t=1}^T \min \left\{ F^{\star}_t - F_t(S_t), B_f^{\theta^{\star}} \right\}  \right] \right]\\
    \nonumber \leq &  2 \delta \cdot \mathbb{E}_{\theta^{\star} \sim \pi} \left[ B_f^{\theta^{\star}} \right] \cdot MT \\
    \nonumber &+  \frac{3(1- 2\delta)}{1-\gamma}  \cdot \left( \mathbb{E}_{\theta^{\star} \sim \pi} \left[ B_f^{\theta^{\star}} \right] + 2 \sqrt{\frac{d}{\lambda}}L \right) \cdot M  \\
    \label{equ:P1} &+ \frac{1-2\delta}{1-\gamma} \cdot \mathbb{E}_{\theta^{\star} \sim \pi} \left[ \operatorname{RGT} \left( 2\beta^{\theta^{\star}}_T\left(\frac{\delta}{2}\right) + 2\alpha_T\left(\frac{\delta}{2}\right)\right) \right].
\end{align}

\subsection{The Second Half of Regret Bound as in (\ref{equ:rgt_p2})}

Conditioned on $E^{\theta^{\star}} \left( \frac{\delta}{2} \right)$, which holds with probability $1- \delta$, $\tilde \theta_t$ lies in a confidence ellipsoid around $\theta^{\star}$ at all times,
\begin{equation*}
    \| \tilde \theta_t - \theta^{\star} \|_{V_{t}} \leq \eta_{t}(\delta), \quad \forall t \in [T].
\end{equation*}
We wrap up an upper bound for $\sum_{t=1}^T \sum_{i=1}^M  \left| \langle \tilde\theta_t -\theta^{\star}, x_{t,i} \rangle \right|$ derived by calling Lemma \ref{lmm:batch_update} into the corollary as follows.

\begin{corollary}
\label{cor:prediction_err}
Conditioned on $E^{\theta^{\star}} \left( \frac{\delta}{2} \right)$, the part of total regret contributed by the prediction error of TS sampled $\tilde\theta_t$ is upper bounded by
\begin{align}
     \sum_{t=1}^T \sum_{i=1}^M  \left| \langle \tilde\theta_t -\theta^{\star}, x_{t,i} \rangle \right| \leq \operatorname{RGT}  \left( \beta^{\theta^{\star}}_T\left(\frac{\delta}{2}\right) + \alpha_T\left(\frac{\delta}{2}\right) \right) .
\end{align}
\end{corollary}

\begin{proof}
Lemma \ref{lmm:batch_update} directly applies by customizing the parameter $\eta_{t}(\delta)$ to be $\beta^{\theta^{\star}}_T\left(\frac{\delta}{2}\right) + \alpha_T\left(\frac{\delta}{2}\right)$.
\end{proof}

With Corollary \ref{cor:prediction_err} ready, we take expectation first conditioned on $\theta^{\star}$ and then over the prior of $\theta^{\star}$, which finally gives an upper bound of (\ref{equ:rgt_p2}) as
\begin{align}
    \nonumber &\mathbb{E}_{\theta^{\star} \sim \pi} \left[ \mathbb{E}_{\theta^{\star}} \left[  \sum_{t=1}^T \sum_{i=1}^M \min \left\{ \left| \langle \tilde\theta_t -\theta^{\star}, x_{t,i} \rangle \right|, B_f^{\theta^{\star}} \right\} \right] \right]\\
    \nonumber \leq &\mathbb{E}_{\theta^{\star} \sim \pi} \left[ \delta M T  B_f^{\theta^{\star}} + (1-\delta) \operatorname{RGT}  \left( \beta^{\theta^{\star}}_T\left(\frac{\delta}{2}\right) + \alpha_T\left(\frac{\delta}{2}\right) \right) \right]\\
    \label{equ:P2} =&  \delta \cdot \mathbb{E}_{\theta^{\star} \sim \pi} \left[ B_f^{\theta^{\star}} \right] \cdot MT + (1-\delta) \cdot \mathbb{E}_{\theta^{\star} \sim \pi} \left[ \operatorname{RGT}  \left( \beta^{\theta^{\star}}_T\left(\frac{\delta}{2}\right) + \alpha_T\left(\frac{\delta}{2}\right) \right) \right].
\end{align}

\subsection{Final Bound: Combining The Two Halves (\ref{equ:rgt_p1}) and (\ref{equ:rgt_p2})}  
Pulling two parts (\ref{equ:P1}) and (\ref{equ:P2}) into the regret decomposition (\ref{rgt_dcp}), for $\forall \delta \in (0,1)$, with $\gamma$ s.t. $\frac{1}{1-\gamma} =O \left( \frac{\sqrt{d}}{\mu} \right) $, the Bayesian regret of Alg.\ref{alg:one-level(PS)} is bounded by
\begin{align*}
     \operatorname{BayesRGT}(T, M)
     \leq&  3 \delta \cdot \mathbb{E}_{\theta^{\star} \sim \pi} \left[ B_f^{\theta^{\star}} \right] \cdot MT \\
     &+ \frac{3(1- 2\delta)}{1-\gamma}  \cdot \left( \mathbb{E}_{\theta^{\star} \sim \pi} \left[ B_f^{\theta^{\star}} \right] + 2 \sqrt{\frac{d}{\lambda}}L \right) \cdot M \\
     &+ (1-\delta) \cdot \mathbb{E}_{\theta^{\star} \sim \pi} \left[ \operatorname{RGT}  \left( \beta^{\theta^{\star}}_T\left(\frac{\delta}{2}\right) + \alpha_T\left(\frac{\delta}{2}\right) \right) \right]\\
     &+ \frac{1-2\delta}{1-\gamma} \cdot \mathbb{E}_{\theta^{\star} \sim \pi} \left[ \operatorname{RGT} \left( 2\beta^{\theta^{\star}}_T\left(\frac{\delta}{2}\right) + 2\alpha_T\left(\frac{\delta}{2}\right)\right) \right].
\end{align*}

By taking the probability of failure $\delta$ to be of $O(\frac{1}{T})$, we finally arrive at
\begin{align}
    \nonumber \operatorname{BayesRGT}(T, M)
    \leq   & O \left(  \frac{1}{1-\gamma} \cdot \left( \mathbb{E}_{\theta^{\star} \sim \pi} \left[ B_f^{\theta^{\star}} \right] + 2 \sqrt{\frac{d}{\lambda}}L \right) \cdot M \right. \\
    \label{equ:RGT_ub} & \quad + \left.\frac{1}{1-\gamma} \cdot \mathbb{E}_{\theta^{\star} \sim \pi} \left[  \operatorname{RGT} \left( 2\beta^{\theta^{\star}}_T\left(\frac{1}{2T}\right) + 2\alpha_T\left(\frac{1}{2T}\right)\right) \right] \right),
\end{align}
where $B_f^{\theta^{\star}} = 2 L \|\theta^{\star}\|$ and  $\frac{1}{1-\gamma}=O \left(  \frac{\sqrt{d}}{\mu} \right)$.

The orders of two expectations in (\ref{equ:RGT_ub}) is claimed as follows.
\begin{claim}
\label{clm:RGT_ord}
The orders of $\mathbb{E}_{\theta^{\star} \sim \pi} \left[ B_f^{\theta^{\star}} \right]$ and $\mathbb{E}_{\theta^{\star} \sim \pi} \left[  \operatorname{RGT} \left( 2\beta^{\theta^{\star}}_T\left(\frac{1}{2T}\right) + 2\alpha_T\left(\frac{1}{2T}\right)\right) \right]$ are:
\begin{itemize}
    \item $\mathbb{E}_{\theta^{\star} \sim \pi} \left[ B_f^{\theta^{\star}} \right]$ is of order
    \begin{equation}
     O \left( \sqrt{\frac{d}{\lambda}} L\right).
    \end{equation}
    \item $\mathbb{E}_{\theta^{\star} \sim \pi} \left[  \operatorname{RGT} \left( 2\beta^{\theta^{\star}}_T\left(\frac{1}{2T}\right) + 2\alpha_T\left(\frac{1}{2T}\right)\right) \right]$ is of order
    \begin{equation}
    O\left( \frac{ L}{\sqrt{\lambda}} d \sqrt{M} (\sqrt{T} + \sqrt{dM}) \log\left(\frac{\sigma^2 \lambda d + T M L^2}{ \sigma^2 \lambda d}\right) \right).
\end{equation}
\end{itemize}
\end{claim}

\begin{proof}
See Appendix \ref{prf_clm:RGT_ord}.
\end{proof}

Therefore, use $\tilde O$ to hide logarithmic term and lower $O(1)$ order term on $T$, recall $\frac{1}{1-\gamma}=O \left(  \frac{\sqrt{d}}{\mu} \right)$ and $L = \sqrt{d}$, we finally arrived at a Bayesian regret of order
\begin{equation}
    \operatorname{BayesRGT}(T, M) = \tilde O \left( \frac{\sqrt{d}}{\mu} \cdot \frac{ L}{\sqrt{\lambda}} \cdot d \sqrt{MT}  \right) = \tilde O \left( \frac{d}{\mu \sqrt{\lambda}} \cdot d \sqrt{MT} \right).
\end{equation}

\section{Real-world experiment validation} 

\subsection{Optimizing sequence fitness for CRISPR gene-editing}
\label{crispr}

Our TS-DE method was adapted for use in a gene-editing application in real-world experiments. 
Briefly, gene-editing, exemplified by technology derived from the Clustered Regularly Interspaced Short Palindromic Repeats, or CRISPR system, is a powerful tool for engineering genetic information in living organisms, and has transformed basic research and human therapeutics \citep{doudna2014new}. The efficiency and outcome of CRISPR gene-editing is highly dependent on the selection of guideRNA sequences, which form a complex with CRISPR proteins to perform gene-editing \citep{shalem2015high}. The TS-DE was applied to guide high-throughput CRISPR gene-editing experiments. In particular, we use known genomic motif features and a linear model for modeling the log editing capacity. At the beginning of each round of experiment, we computationally generate a new library of design sequences by randomly generating mutations and recombinations based on the previous population. Then we apply the bandit linear model to select sequences with high predicted fitness, and evaluate their actual editing capacities in the next round of experiments. A total of ~14,358 unique guideRNA sequences were measured, and the log capacity improved by $\approx 5$. Notably, the optimized CRISPR designs generated by our DE approach is part of another manuscript ({\it Hughes NW, Zhang J, Pierce J, Qu Y, Wang C, Agrawal A., Morri M, Neff N, Winslow MM, Wang M, and Cong L. Machine Learning Optimized Cas12a Barcoding Enables Recovery of Single-Cell Lineages and Transcriptional Profiles. Molecular Cell. 2022. In Press.}), demonstrating real-world utility of current method. See Fig.\ref{fig:S_f_exp} (borrowed from Hughes et al, 2022) for an illustration of the pipeline.  We refer to Hughes et al, 2022 for more details on the experiment and computation. 

\paragraph{Remark} The above real-world application of bandit DE differs from Algorithm \ref{alg:one-level(PS)} and generalizes it in a number of ways. For example, features used for predicting the gene-editing efficiency are not limited to motif features. Also they are not binary valued. Second, recombination and mutation were not done exactly as in Modules \ref{mod:rcb&slc} and \ref{mod:dire_mutation}. They were randomized on the basepair level rather than the motif level. Despite these differences, our method was able to guide the experiment and accelerate discovery. 
This demonstrates the bandit DE method may have broad generalizability and it is not restricted to the abstract mathematical model formulated in this paper. 

\begin{figure}[h]
    \centering
    \includegraphics[width=\linewidth]{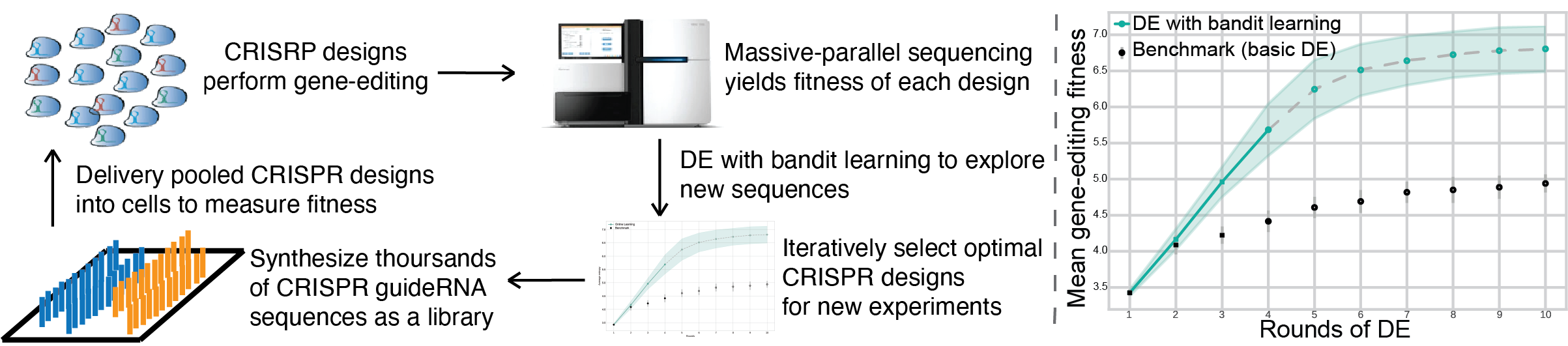}
    \caption{\textbf{Evolving CRISPR sequences using iterative real-world experiments and accelerated DE}  Left panels: Workflow overview. Right panel: Fitness distribution showing accelerated optimization using DE with Bandit learning. (Ths figure is borrowed from Hughes et al, 2022)
    \label{fig:S_f_exp}
    }
\end{figure}

\section{Proof of Theorem \ref{thm:ascent}}

\subsection{Ascent property of \texttt{Crossover\_Selection}}
\label{app-ascent-general}
\begin{proof}
Since each $z \in S^{\prime}$ is generated in the same way independently and $F(S^{\prime})$ is the fitness averaging over all $z$'s, thus
\begin{equation*}
    \mathbb{E}\left[ F(S^{\prime}) \right] = \mathbb{E}\left[ \left. f(z) \right \vert z \in S^{\prime}  \right],  
\end{equation*}
with the expectation taken over the randomness in sampling $z$'s parents $x$ and $y$ and in crossing over $x$ and $y$. Using notation $\mathbb{E}_{x,y} \left[ \cdot \right]: = \mathbb{E}\left[ \left. \cdot \right \vert x,y \right]$, the conditional expectation given $x$ and $y$, rewrite $\mathbb{E}\left[ \left. f(z) \right \vert z \in S^{\prime} \right]$ as
\begin{equation*}
    \mathbb{E}\left[ \left. f(z) \right \vert z \in S^{\prime} \right] = \mathbb{E}\left[ \mathbb{E}_{x, y} \left[ f(z) \left \vert f(z) \geq \frac{f(x) + f(y)}{2} \right.  \right] \right],
\end{equation*}
where the inner expectation is taken over the randomness in the recombination step $z \leftarrow \texttt{Rcb}(x, y)$, and the outer expectation is over sampling $x$ and $y$.

Given $x$ and $y$, a recombined child sample $z$ can be represented by
\begin{equation}
    z = \frac{x + y}{2} + \frac{x - y}{2} \cdot e,
\end{equation}
where the $\cdot$ operator here multiplies two vector entrywisely into a new vector and $e$ is a vector consisting of $d$ independent Rademacher variables, that is $e = (e_i, \cdots, e_d)^{\top}$ and
\begin{equation*}
    e_i \overset{\text{i.i.d.}}{\sim} \operatorname{Rad}.
\end{equation*} 
Thus $f(z)$ is computed as 
\begin{equation}
    f(z) = \frac{f(x) + f(y)}{2} + \frac{1}{2} \sum_{i=1}^{d} \theta_i \left( x_i - y_i \right) e_i.
\end{equation}
And then $f(z) \geq \frac{f(x) + f(y)}{2}$ is equivalent to $\sum_{i=1}^{d} \theta_i \left( x_i - y_i \right) e_i \geq 0$, so
\begin{align}
    \nonumber &\mathbb{E}_{x,y} \left[ f(z) \left \vert f(z) \geq \frac{f(x) + f(y)}{2} \right.  \right] \\
    \nonumber =& \frac{f(x) + f(y)}{2} + \frac{1}{2} \mathbb{E} \left[  \sum_{i=1}^{d} \theta_i \left( x_i - y_i \right) e_i \left \vert \sum_{i=1}^{d} \theta_i \left( x_i - y_i \right) \geq 0 \right.  \right]\\
    \label{equ:symmetry}& \frac{f(x) + f(y)}{2} + \frac{1}{2} \mathbb{E} \left[  \left|\sum_{i=1}^{d} \theta_i \left( x_i - y_i \right) e_i \right| \right]\\
    \nonumber \geq& \frac{f(x) + f(y)}{2} + \frac{C}{2} \| \theta \cdot \left( x - y \right)\|,
\end{align}
where (\ref{equ:symmetry}) holds because $\sum_{i=1}^{d} \theta_i \left( x_i - y_i \right) e_i$ is symmetrically distributed around $0$. And in the last line, $\cdot$ is still the entrywise multiplication between vectors and $C \geq \frac{1}{\sqrt{2}}$ according to \citet{haagerup1981best}.

Thus, 
\begin{align}
    \nonumber \mathbb{E}\left[ F(S^{\prime}) \right] &=  \mathbb{E}\left[ \mathbb{E}_{a,b} \left[ f(z) \left \vert f(z) \geq \frac{f(x) + f(y)}{2} \right.  \right] \right]\\
    \nonumber &\geq  \mathbb{E}\left[ \frac{f(x) + f(y)}{2} \right] + \frac{1}{2 \sqrt{2}}  \mathbb{E}\left[ \| \theta \cdot \left( x - y \right)\| \right]\\
    &\geq F(S) + \frac{1}{2 \sqrt{2}}  \mathbb{E}\left[ \| \theta \cdot \left( x - y \right)\| \right].
\end{align}

By Cauchy-Schwarz, we have
\begin{equation*}
    \| \theta \cdot \left( x - y \right)\| \geq  \frac{1}{\sqrt{d}} \sum_{i=1}^d |\theta_i| |x_i - y_i|
\end{equation*}
Thus, by averaging over all $x$ and $y$ sampled from $S$, 
\begin{equation*}
    \mathbb{E}\left[ \| \theta \cdot \left( x - y \right)\| \right] \geq  \frac{1}{\sqrt{d}} \sum_{i=1}^d |\theta_i| \mathbb{E}\left[  |x_i - y_i| \right]
\end{equation*}
 
When $\forall i \in [d], x_i, y_i \in \{0,1
\}$ for all $x$ and $y$'s in $S$, then
\begin{align}
    &\nonumber \mathbb{E}\left[  |x_i - y_i| \right] \geq \mathbb{E}\left[  \left(x_i - y_i \right)^2 \right] = 2 Var_i(S),\\
    &\mathbb{E}\left[ \| \theta \cdot \left( x - y \right)\| \right] \geq \frac{2}{\sqrt{d}} \sum_{i=1}^d |\theta_i| Var_i(S),
\end{align}
where $Var_i(S)$ denotes the variance of $x_i$ when $x$ is uniformly sampled from $S$.

Therefore,
\begin{equation*}
    \mathbb{E}\left[ F(S^{\prime}) \right]  \geq F(S) + \frac{1}{\sqrt{2d}} \sum_i |\theta_i| Var_i(S).
\end{equation*}
\end{proof}

\section{Omitted Proofs in Appendix \ref{proof_main}}

\subsection{Proof of Lemma \ref{lmm:DM}}
\label{prf_lmm:DM}

\begin{proof}
For $\forall i \not \in \mathcal{I}$, $\forall x \in S$ is not induced to mutate at site $i$, thus for $x^{\prime} = \texttt{Mut}(x, \mathcal{I}, \mu)$, $x^{\prime}_i = x_i$ and
\begin{equation*}
    p_i(S^{\prime}) =  p_i(S).
\end{equation*}

For $i \in \mathcal{I}$, after the directed mutation formulated as (\ref{equ:d_mutation}), $\mathbb{E} \left[ \mathbf{I} \left\{ x^{\prime}_i = 1 \right\} \right] = (1-\mu)\mathbf{I} \left\{ x_i = 1 \right\} + \frac{\mu}{2}$.
\begin{equation*}
     \mathbb{E} \left[ p_i(S^{\prime}) \right] =  (1-\mu)p_i(S) + \frac{\mu}{2}
    = p_i(S) + \left( \frac{1}{2} - p_i(S)\right) \mu. 
\end{equation*}

Since $i \in \mathcal{I}$ iff $\frac{1}{M} \sum_{x\in S} \theta_i  \cdot x_i \leq  \theta_i \cdot \bar x_i$, which is equivalent to $p_i(S) \leq \frac{1}{2}$, showing that the $i$-th dimension is not sufficient with the favored feature. Then the directed mutation strictly increases $p_i(S)$ for any insufficient dimension $i$ by boosting it by $\mu \left( \frac{1}{2} - p_i(S) \right) \geq 0$, which resulting in a $|\theta_i|$-increase in the utility value per unit of increase in $p_i(S)$.

Therefore, $\mathbb{E} \left[ F(S^{\prime})\right] \geq F(S)$ and
\begin{align*}
    \mathbb{E} \left[ p_i(S^{\prime}) \right]
    &= p_i(S) > \frac{1}{2}, \quad \forall i \not \in \mathcal{I},\\
    \mathbb{E} \left[ p_i(S^{\prime}) \right]
    &= p_i(S) + \left( \frac{1}{2} - p_i(S)\right) \mu \geq \frac{\mu}{2}, \quad \forall i \in \mathcal{I}.
\end{align*}
Thus, after calling $S^{\prime} = \texttt{Directed\_Mutation}(f, S, \mu)$, $\mathbb{E} \left[ p_i(S^{\prime}) \right] \geq \frac{\mu}{2}, \forall i \in [d]$. By a standard argument of concentration and a union bound taken over $i \in [d]$, with probability $1-\delta$,
\begin{equation*}
    p_i(S^{\prime}) \geq \frac{\mu}{4}, \quad \forall i \in [d]
\end{equation*}
when $|S| = \Omega \left(  \frac{\log(\frac{d}{\delta})}{\mu^2} \right)$.
\end{proof}

\subsection{Proof of Lemma \ref{lmm:lb_min_pi}}
\label{prf_lmm:lb_min_pi}
\begin{proof}
Lemma \ref{lmm:lb_min_pi} is derived by taking union bound over $t+1 \in [T]$ upon $\min_i p_i^{\tilde \theta_{t+1}}  \left( S^{\prime}_t \right) \geq \frac{\mu}{4}$ obtained by instantiating (\ref{equ:dm_feature_lb}) for $S^{\prime}_{t}$ and $f_{\tilde \theta_{t+1}}$ in Lemma \ref{lmm:DM}.
\end{proof}

\subsection{Proof of Lemma \ref{lmm:LC_under_conditions}}
\label{prf_lmm:LC_under_conditions}
\begin{proof}
Recall from Theorem \ref{thm:ascent} that
\begin{equation}
\tag{\ref{equ:ascent} revisited}
    \mathbb{E}\left[ F(S^{\prime}) \right] \geq F(S) + \frac{1}{\sqrt{2d}} \sum_i |\theta_i| \textrm{Var}_i(S),
\end{equation}
where $Var_i(S)$ is the variance of $x_i$ when $x$ is uniformly sampled from $S$. Using $p_i(S)$ defined in Definition \ref{def:ratio_binary}
\begin{equation}
    \textrm{Var}_i(S) = p_i(S) \left( 1-p_i(S) \right).
\end{equation}

Then, it suffices to prove 
\begin{equation*}
    \sum_{i=1}^d |\theta_i|  p_i(S) \left( 1-p_i(S) \right) \geq \min_i p_i(S) \cdot \left( F^{\star} - F(S) \right).
\end{equation*}
Taking a closer look at the suboptimality gap $F^{\star} - F(S)$, it is easily observed that
\begin{align}
    \label{equ:F*_binary} F^{\star} &= \sum_{i: \theta_i \geq 0} \theta_i + \sum_{i: \theta_i < 0} 0,\\
    \nonumber F(S) &= \sum_{i: \theta_i \geq 0} \theta_i \left[ p_i(S) \cdot 1 + (1-p_i(S)) \cdot 0 \right] + \sum_{i: \theta_i < 0} \theta_i \left[ p_i(S) \cdot 0 + (1-p_i(S)) \cdot 1 \right]\\
    \label{equ:F(S)_binary} &= \sum_{i: \theta_i \geq 0} \theta_i \cdot p_i(S) + \sum_{i: \theta_i < 0} \theta_i \cdot (1-p_i(S)).
\end{align}
Plugging in (\ref{equ:F*_binary}) and (\ref{equ:F(S)_binary}), we have
\begin{equation}
\label{equ:sub_gap}
    F^{\star} - F(S) = \sum_{i} |\theta_i| (1 - p_i(S)).
\end{equation}

Therefore,
\begin{equation*}
    \sum_{i=1}^d |\theta_i|  p_i(S) \left( 1-p_i(S) \right) \geq \min_i p_i(S) \cdot \left( F^{\star} - F(S) \right).
\end{equation*}
\end{proof}

\subsection{Proof of Lemma \ref{lmm:hp_events}}
\label{prf_lmm:hp_events}
\begin{proof}
We finish the proof by lower bounding the probabilities of two events $\hat{E}^{\theta^{\star}} \left( \frac{\delta}{2} \right)$ and $\tilde{E}^{\theta^{\star}} \left( \frac{\delta}{2} \right)$ by $1-\frac{\delta}{2}$ separately. Recall that for $\forall t\in [T]$
\begin{align}
    \tag{\ref{equ:TS_dist} revisited} \tilde{\theta}_{t} &\sim \mathcal{N}(\hat \theta_{t}, V^{-1}_{t}),\\
    \tag{\ref{equ:cov_matrix} revisited} V_{t} &= \frac{1}{\sigma^2} \Phi^{\top}_{t-1} \Phi_{t-1} + \lambda I,\\
    \tag{\ref{equ:hat_theta} revisited} \hat \theta_{t} &= \frac{1}{\sigma^2} V^{-1}_{t} \Phi^{\top}_{t-1} U_{t-1}. 
\end{align}

\textbf{Bounding $\mathbb{P}\left(  \hat{E}^{\theta^{\star}} \left( \frac{\delta}{2} \right)  \right)$.} 

Plugging (\ref{equ:cov_matrix}) into (\ref{equ:hat_theta}), we will see $\hat \theta_{t}$ is related to the regularized least square estimator (RLS):
\begin{equation*}
    \hat \theta_{t} = \frac{1}{\sigma^2} V^{-1}_{t} \Phi^{\top}_{t-1} U_{t-1} = \left( \Phi^{\top}_{t-1} \Phi_{t-1} + \sigma^2 \lambda I \right)^{-1} \Phi^{\top}_{t-1} U_{t-1}.
\end{equation*}
For any fixed ground truth $\theta^{\star}$, $\hat \theta_{t}$ is a RLS estimator of $\theta^{\star}$ regularized by $\sigma^2 \lambda \cdot I$. 
Conditioned on $\theta^{\star}$, define a filtration w.r.t. the data $\{\left( x_{t,i}, u(x_{t,i})\right), i \in [M], t \in [T-1] \}$ collected along the way.

\begin{definition}
Define $\mathcal{F}_{t}$ be the information accumulated after the $t$-th batch of data points is collected.
\begin{align}
    \mathcal{F}_{0}: &= \sigma \left( \theta^{\star} \right)\\
    \mathcal{F}_{t}: &= \{\mathcal{F}_{t-1}, \sigma \left( x_{t,1}, u(x_{t,1}), \cdots, x_{t,M}, u(x_{t,M}) \right)  \}.
\end{align}
Then we fine grind the filtration $\{F_t\}_{t=0}^{T-1}$ to be
\begin{equation}
    \mathcal{F}_{0} \subset \mathcal{F}_{0,1} \subset \cdots \subset \mathcal{F}_{0,M}\subset \mathcal{F}_{1}\subset \cdots \subset \mathcal{F}_{t-1}\subset \mathcal{F}_{t-1,1}\subset \cdots \mathcal{F}_{t-1,M}\subset \mathcal{F}_{t-1}\subset \cdots \subset \mathcal{F}_{T-1}
\end{equation}
by essentially adding $M$ layers between $\mathcal{F}_{t-1}$ and $\mathcal{F}_{t}$ and each layer $\mathcal{F}_{t-1,i}$ contains the information obtained after $(x_{t,i}, u(x_{t,i}))$ is added to the dataset.
\end{definition}

Under Assumption \ref{asmp_noise}, each feedback $u(x_{t,i})$ satisfies
\begin{equation}
    u(x_{t,i}) = f_{\theta^{\star}}(x_{t,i}) + \xi_{t,i}, \tag{\ref{equ:noise} revisited}
\end{equation}
where
\begin{equation*}
    \left. \xi_{t,i} \right\vert \mathcal{F}_{t,i} \sim \mathcal{N} \left(0, \sigma^2 \right).
\end{equation*}

Then bounding $\mathbb{P}\left(  \hat{E}^{\theta^{\star}} \left( \frac{\delta}{2} \right)  \right)$ is a straightforward application of the Theorem 2 in \cite{abbasi2011improved}, wrapped up into the following proposition.

\begin{proposition}
\label{pop:beta_radius}
Under Assumption \ref{asmp_noise}, for $\forall \delta \in (0,1)$ and any $\mathcal{F}_{t,i}$-adapted data sequence $\left( \{x_{0,i}\}_{i = 1}^M , \cdots, \{x_{T-1,i}\}_{i = 1}^M \right)$ s.t. $\|x_{t,i}\| \leq L$,
\begin{equation}
    \mathbb{P}\left( \left. \exists t \in [T]: \left\|\hat \theta_{t}-\theta^{\star}\right\|_{V_{t}} \geq \beta_{t}\left(\delta \right) \right \vert \mathcal{F}_{0}\right) \leq \delta.
\end{equation}
\end{proposition}

From the result above, we have 
\begin{align*}
    \mathbb{P}\left(  \hat{E}^{\theta^{\star}} \left( \frac{\delta}{2} \right)  \right) &= \mathbb{P}\left(\left. \left\|\hat \theta_{t}-\theta^{\star}\right\|_{V_{t}} \leq \beta_{t}\left( \frac{\delta}{2} \right), \forall t \in [T] \right\vert \theta^{\star} \right)\\
    &= 1- \mathbb{P}\left( \left. \exists t \in [T]: \left\|\hat \theta_{t}-\theta^{\star}\right\|_{V_{t}} \geq \beta_{t}\left( \frac{\delta}{2} \right)\right\vert \theta^{\star}\right)\\
    &\geq 1-\frac{\delta}{2}.
\end{align*}

\textbf{Bounding $\mathbb{P}\left(  \hat{E}^{\theta^{\star}} \left( \frac{\delta}{2} \right)  \right)$.} Recall that $\tilde \theta_{t}$ is sampled from posterior distribution $\mathcal{N}(\mathbf{\hat \theta}_{t}, V^{-1}_{t})$ independently from $\theta^{\star}$, then we have
\begin{equation}
    \left\|\tilde{\theta}_{t}-\hat \theta_{t}\right\|^2_{V_{t}} =  \left\|V_{t}^{\frac{1}{2}} (\tilde{\theta}_{t}-\hat \theta_{t}) \right\|^2, \quad \forall t \in [T]
\end{equation}
where $V_{t}^{\frac{1}{2}} (\tilde{\theta}_{t}-\hat \theta_{t}) \sim \mathcal{N}(0, \mathbf{I})$. Thus $\left\|\tilde{\theta}_{t}-\hat \theta_{t}\right\|^2_{V_{t}} \sim \chi^2_d$ independently from $\theta^{\star}$. From the concentration of $\chi^2_d$ random variable, we have
\begin{equation*}
    \mathbb{P}\left( \chi_d \geq 2\sqrt{d \log \left( \frac{1}{\delta} \right) } + \sqrt{d} \right) \leq \delta.
\end{equation*}
Therefore, by taking a union bound over $\forall t \in [T]$, we have
\begin{align*}
    \mathbb{P}\left( \left. \tilde{E} \left( \frac{\delta}{2} \right) \right \vert \theta^{\star} \right) &= \mathbb{P}\left( \tilde{E} \left( \frac{\delta}{2} \right) \right)\\
    &= \mathbb{P}\left( \left\|\tilde{\theta}_{t}-\hat \theta_{t}\right\|_{V_{t}} \leq \alpha_{t}\left( \frac{\delta}{2} \right), \forall t \in [T] \right)\\
    &\geq 1- \sum_{t=1}^{T} \mathbb{P}\left( \left\|\tilde{\theta}_{t}-\hat \theta_{t}\right\|_{V_{t}} \leq \alpha_{t}\left( \frac{\delta}{2} \right) \right)\\
    &\geq 1- \sum_{t=1}^{T} \frac{\delta}{2T} = 1- \frac{\delta}{2}.
\end{align*}
\end{proof}

\subsection{Proof of Proposition \ref{pop:beta_radius}}
\begin{proof}
Use notation $\tilde V_t:= \Phi^{\top}_{t-1} \Phi_{t-1} + \sigma^2 \lambda I $, then $V_t = \frac{1}{\sigma^2} \tilde V_t$ and $\hat \theta_{t} = {\tilde V_t }^{-1} \Phi^{\top}_{t-1} U_{t-1}$.
According to Theorem 2 in \cite{abbasi2011improved}, for $\forall \delta \in (0,1)$ and any $\mathcal{F}_{t,i}$-adapted data sequence $\left( \{x_{0,i}\}_{i = 1}^M , \cdots, \{x_{T-1,i}\}_{i = 1}^M \right)$ s.t. $\|x_{t,i}\| \leq L$,
\begin{equation*}
    \mathbb{P}\left( \left. \exists t \in [T]: \left\|\hat \theta_{t}-\theta^{\star}\right\|_{\tilde V_{t}} \geq \sigma \cdot \beta_{t}\left(\delta \right) \right \vert \mathcal{F}_{0}\right) \leq \delta.
\end{equation*}
Therefore, since $\left\|\hat \theta_{t}-\theta^{\star}\right\|_{V_{t}} = \left\|\hat \theta_{t}-\theta^{\star}\right\|_{\frac{1}{\sigma^2} \tilde V_{t}} = \frac{1}{\sigma} \left\|\hat \theta_{t}-\theta^{\star}\right\|_{\tilde V_{t}} $
\begin{equation*}
    \mathbb{P}\left( \left. \exists t \in [T]: \left\|\hat \theta_{t}-\theta^{\star}\right\|_{V_{t}} \geq \beta_{t}\left(\delta \right) \right \vert \mathcal{F}_{0}\right) \leq \delta.
\end{equation*}
\end{proof}

\subsection{Proof of Lemma \ref{lmm:batch_update}}
\label{prf_lmm:batch_update}
\begin{proof}
We are about to take a closer look at the incremental increase of the determinant of $V_t$, define $V_{t, l} = \frac{1}{\sigma^2} \left( \sigma^2 \lambda I + \sum_{i=1}^{t-1}\sum_{j=1}^M x_{i,j} x_{i,j}^T + \sum_{j=1}^{l} x_{t,j} x_{t,j}^T \right)$ for $\forall t \in [T], l \in [M]$ and thus $V_{t, 0} = V_t $ .
Mark the time steps where $V_t$ has significant increase in its determinant by $ C := \{t \in [T]: \frac{\det(V_{t+1})}{\det(V_{t})} > 2\}$. Then the prediction errors in $T$ rounds can be divided into two parts as
\begin{equation}
\label{equ:pre_err_div}
    \sum_{t=1}^T \sum_{i=1}^M \left| \langle a_t - b_t, x_{t,i} \rangle \right| = \sum_{t \notin C} \sum_{i=1}^M \left| \langle a_t - b_t, x_{t,i} \rangle \right| + \sum_{t \in C} \sum_{i=1}^M \left| \langle a_t - b_t, x_{t,i} \rangle \right|.
\end{equation}

The first half of (\ref{equ:pre_err_div}) consists of error accumulated in the rounds where $\det(V_{t})$ didn't increased much after having a batch update of size $M$, so we bound this part in the same spirit of bounding the case where only rank-$1$ update happens per round. Result is stated in the following claim.
\begin{claim}
\label{clm:self_1}
The first half of (\ref{equ:pre_err_div}) is bounded by
\begin{equation}
    \sum_{t \notin C} \sum_{i=1}^M   \left| \langle a_t - b_t, x_{t,i} \rangle \right| \leq  \eta_T(\delta) \sqrt{ \frac{2L^2 + 2\lambda}{\lambda}}  \sqrt{ MT \log\left(\frac{\det \left(V_{T+1}\right)}{\det \left(V_1\right)}\right)}.
\end{equation} 
\end{claim}

For the second half of (\ref{equ:pre_err_div}), we are about to bound by showing $|C|$ is small. Notice that
\begin{equation}
\frac{\det \left(V_{T+1}\right)}{\det \left(V_{1}\right)} = \prod_{t=1}^{T} \frac{\det \left(V_{t+1}\right)}{\det \left(V_{t}\right)} \geq \prod_{t \in C} \frac{\det \left(V_{t+1}\right)}{\det \left(V_{t}\right)} \geq 2^{|C|},
\end{equation}
thus $|C|$ should not be greater than $2\log \left(\frac{\det \left(V_{T+1}\right)}{\det \left(V_{1}\right)}\right)$. And for $\forall t \in [T],i\in [M]$
\begin{equation}
    \left| \langle a_t - b_t, x_{t,i} \rangle \right| \leq \| a_t - b_t \|_{V_t} \|x_{t,i}\|_{V_t^{-1}} \leq  \frac{\eta_T(\delta) L}{\sqrt{\lambda}}.
\end{equation}

Putting two parts together, we have
\begin{align*}
    \sum_{t=1}^T \sum_{i=1}^M \left| \langle a_t - b_t, x_{t,i} \rangle \right| =& \sum_{t \notin C} \sum_{i=1}^M \left| \langle a_t - b_t, x_{t,i} \rangle \right| + \sum_{t \in C} \sum_{i=1}^M \left| \langle a_t - b_t, x_{t,i} \rangle \right|\\
    \leq& \eta_T(\delta) \sqrt{ \frac{2L^2 + 2\lambda}{\lambda}}  \sqrt{ MT \log\left(\frac{\det \left(V_{T+1}\right)}{\det \left(V_1\right)}\right)} + \eta_T(\delta) \frac{ L}{\sqrt{\lambda}} |C| M \\
    \leq& \eta_T(\delta) \sqrt{ \frac{2L^2 + 2\lambda}{\lambda}}  \sqrt{ MT \log\left(\frac{\det \left(V_{T+1}\right)}{\det \left(V_1\right)}\right)}  + \eta_T(\delta) \frac{ 2L}{\sqrt{\lambda}} M \log \left(\frac{\det \left(V_{T+1}\right)}{\det \left(V_{1}\right)}\right)\\
    \leq& \eta_T(\delta) \sqrt{ \frac{2L^2 + 2\lambda}{\lambda}}  \sqrt{ dMT \log\left(\frac{\sigma^2 d \lambda + M T L^2}{\sigma^2 d \lambda}\right)}  + \eta_T(\delta) \frac{ 2L}{\sqrt{\lambda}} dM \log\left(\frac{\sigma^2 d \lambda + M T L^2}{\sigma^2 d \lambda}\right).
\end{align*}
where the final line is referring to the result in \cite{abbasi2011improved} that
\begin{equation*}
    \log \left(\frac{\det \left(V_{T+1}\right)}{\det \left(V_{1}\right)}\right) \leq d\log\left(\frac{\sigma^2 d \lambda + M T L^2}{\sigma^2 d \lambda}\right).
\end{equation*}

\end{proof}

\subsection{Proof of Claim \ref{clm:self_1}}
\begin{proof}
With probability $1-\delta$, for $\forall t \in [T]$, normalized by $V_t$, $a_t$ and $b_t$ concentrate around each other with in a radius of $\eta_t(\delta)$, thus
\begin{equation*}
    \left| \langle a_t - b_t, x_{t,i} \rangle \right| \leq \| a_t - b_t \|_{V_t} \|x_{t,i}\|_{V_t^{-1}} \leq \eta_t(\delta) \|x_{t,i}\|_{V_t^{-1}}.
\end{equation*}

Summing over all individuals in time steps $t \notin C$ , we have
\begin{align}
    \nonumber\sum_{t \notin C} \sum_{i=1}^M  \left| \langle a_t - b_t, x_{t,i} \rangle \right| &\leq  \sum_{t \notin C} \sum_{i=1}^M  \eta_t(\delta) \|x_{t,i}\|_{V_t^{-1}} \\
    \nonumber &\leq \eta_T(\delta) \sum_{t \notin C} \sum_{i=1}^M   \|x_{t,i}\|_{V_t^{-1}} \\
    &\nonumber \leq \eta_T(\delta)  \sqrt{MT\sum_{t \notin C} \sum_{i=1}^M \|x_{t,i}\|^2_{V_t^{-1}} }\\
    &\label{equ:x_log1+x}\leq \eta_T(\delta)  \sqrt{\frac{L^2 + \lambda}{\lambda} \cdot  MT\sum_{t \notin C} \sum_{i=1}^M  \log \left(1+\|x_{t,i}\|^2_{V_t^{-1}} \right) },
\end{align}
where (\ref{equ:x_log1+x}) holds because 
\begin{equation*}
    \|x_{t,i}\|^2_{V_t^{-1}} \leq \lambda_{\max}(V_t^{-1}) \|x_{t,i}\|^2 \leq \frac{L^2}{\lambda}.
\end{equation*}

Continuing from (\ref{equ:x_log1+x}), we can substitute the normalization matrix $V_t^{-1}$ with $V_{t,i}^{-1}$, at the cost of inflating by $2$, and then following the classic self-normalized bound on data points. Recall the Lemma 12 in \cite{abbasi2011improved}:
\begin{equation}
    \frac{\left\Vert x \right\Vert^2_{\bA}}{\left\Vert x \right\Vert^2_{\bB}} \leq \frac{\det(\bA)}{\det(\bB)}, \quad \text{if } \bA \succeq \bB.
\end{equation}

Substituting $V_t^{-1}$ with $V_{t,i-1}^{-1}$ in $\|x_{t,i}\|^2_{V_t^{-1}}$, noticing $\frac{\det(V_t^{-1})}{\det(V_{t,i-1}^{-1})} = \frac{\det(V_{t,i-1})}{\det(V_t)} \leq \frac{\det(V_{t+1})}{\det(V_t)} \leq 2$ when $t \not \in C$,  leads to
\begin{align*}
    \|x_{t,i}\|^2_{V_t^{-1}} &\leq 2 \|x_{t,i}\|^2_{V_{t,i}^{-1}},\\
    \log \left(1 + \|x_{t,i}\|^2_{V_t^{-1}} \right) &\leq  \log \left( 1 + 2\|x_{t,i}\|^2_{V_{t,i-1}^{-1}} \right)\\
    &\leq  2 \log \left( 1 + \|x_{t,i}\|^2_{V_{t,i-1}^{-1}} \right).
\end{align*}
Then it follows the self-normalized bound in \cite{abbasi2011improved} and gives that
\begin{align*}
  \sum_{t \notin C} \sum_{i=1}^M \log(1+\|x_{t,i}\|^2_{V_t^{-1}})  &\leq 2 \sum_{t \notin C} \sum_{i=1}^M \log \left( 1 + \|x_{t,i}\|^2_{V_{t,i-1}^{-1}} \right) \\
  &\leq 2 \sum_{t = 1}^T \sum_{i=1}^M \log \left( 1 + \|x_{t,i}\|^2_{V_{t,i-1}^{-1}} \right) \\
  &\leq 2\log \left(\frac{\det \left(V_{T+1}\right)}{\det \left(V_1\right)}\right).
\end{align*}

Therefore, the first half of (\ref{equ:pre_err_div}) is bounded by
\begin{equation*}
    \sum_{t \notin C} \sum_{i=1}^M  \left| \langle a_t - b_t, x_{t,i} \rangle \right| \leq   \eta_T(\delta) \sqrt{ \frac{2L^2 + 2\lambda}{\lambda}}  \sqrt{ MT \log\left(\frac{\det \left(V_{T+1}\right)}{\det \left(V_1\right)}\right)}.
\end{equation*}
\end{proof}

\subsection{Proof of Lemma \ref{lmm:E_theta_t_l2}}
\label{prf_lmm:E_theta_t_l2}
\begin{proof}
By triangle inequality
\begin{equation*}
    \mathbb{E}_{E^{\theta^{\star}}} \left[ \| \tilde \theta_t \| \right] \leq  \| \theta^{\star} \| + \mathbb{E}_{E^{\theta^{\star}}} \left[ \| \tilde \theta_t - \hat \theta_t \| \right] + \mathbb{E}_{E^{\theta^{\star}}} \left[ \| \hat \theta_t - \theta^{\star} \| \right] .
\end{equation*}

Along with $\lambda_{\min} (V_t) \geq \lambda$ since $V_t$ is regularized with $\lambda I$ in its definition, then we have
\begin{equation*}
    \mathbb{E}_{E^{\theta^{\star}}} \left[ \| \tilde \theta_t \| \right] \leq  \| \theta^{\star} \| + \frac{1}{\sqrt{\lambda}} \mathbb{E}_{E^{\theta^{\star}}} \left[ \| \tilde \theta_t - \hat \theta_t \|_{V_t} \right] + \frac{1}{\sqrt{\lambda}} \mathbb{E}_{E^{\theta^{\star}}} \left[ \| \hat \theta_t - \theta^{\star} \|_{V_t} \right] .
\end{equation*}

And in event $E^{\theta^{\star}}:= E^{\theta^{\star}} \left( \frac{\delta}{2} \right) \cap E^{\theta^{\star}}_{\operatorname{DM}}$, $E^{\theta^{\star}}_{\operatorname{DM}}$ is independent from the sampling of $\tilde \theta_t$, and conditioned on event $E^{\theta^{\star}} \left( \frac{\delta}{2} \right)$, both $\| \tilde \theta_t - \hat \theta_t \|_{V_t}$ and $\| \hat \theta_t - \theta^{\star} \|_{V_t}$ tend to be smaller than it is unconditionally. Thus, we lift the condition on $E^{\theta^{\star}}$ to get an upper bound as
\begin{equation}
\label{equ:TS_l2_ub}
    \mathbb{E}_{E^{\theta^{\star}}} \left[ \| \tilde \theta_t \| \right] \leq  \| \theta^{\star} \| + \frac{1}{\sqrt{\lambda}} \mathbb{E}_{\theta^{\star}} \left[ \| \tilde \theta_t - \hat \theta_t \|_{V_t} \right] + \frac{1}{\sqrt{\lambda}} \mathbb{E}_{\theta^{\star}} \left[ \| \hat \theta_t - \theta^{\star} \|_{V_t} \right] .
\end{equation}
Recall that conditioned on any realization of $\theta^{\star}$, $\tilde \theta_t$ is sampled from
\begin{align}
    \tag{\ref{equ:TS_dist} revisited} \tilde{\theta}_{t} &\sim \mathcal{N}(\hat \theta_{t}, V^{-1}_{t})
\end{align}
with 
\begin{align}
    \tag{\ref{equ:cov_matrix} revisited} V_{t} &= \frac{1}{\sigma^2} \Phi^{\top}_{t-1} \Phi_{t-1} + \lambda I,\\
    \tag{\ref{equ:hat_theta} revisited} \hat \theta_{t} &= \frac{1}{\sigma^2} V^{-1}_{t} \Phi^{\top}_{t-1} U_{t-1}. 
\end{align}
So $\left\|\tilde{\theta}_{t}-\hat \theta_{t}\right\|^2_{V_{t}} \sim \chi^2_d$ independent from $\theta^{\star}$ and thus
\begin{equation}
\label{equ:TS_l_Vt}
    \mathbb{E}_{\theta^{\star}} \left[ \| \tilde \theta_t - \hat \theta_t \|_{V_t} \right] \leq \sqrt{d}.
\end{equation}

Also, from (\ref{equ:hat_theta}), let $U_{t-1} = \Phi_{t-1} \theta^{\star} + \xi_{t-1}$ and $\xi_{t-1}$ be the corresponding noise vector, then $\hat \theta_t - \theta^{\star}$ is computed as
\begin{align*}
    \hat \theta_t - \theta^{\star} =& \frac{1}{\sigma^2} V^{-1}_{t} \Phi^{\top}_{t-1} U_{t-1} - \theta^{\star}\\
    \overset{\ref{equ:cov_matrix}}{=}& \left( \Phi^{\top}_{t-1} \Phi_{t-1} + \sigma^2 \lambda I \right)^{-1} \Phi^{\top}_{t-1} U_{t-1} - \theta^{\star}\\
    =& \left( \Phi^{\top}_{t-1} \Phi_{t-1} + \sigma^2 \lambda I \right)^{-1} \Phi^{\top}_{t-1} (\Phi_{t-1} \theta^{\star} + \xi_{t-1}) - \theta^{\star}\\
    =& \left( \Phi^{\top}_{t-1} \Phi_{t-1} + \sigma^2 \lambda I \right)^{-1} \Phi^{\top}_{t-1}\xi_{t-1} - \sigma^2 \lambda \left( \Phi^{\top}_{t-1} \Phi_{t-1} + \sigma^2 \lambda I \right)^{-1} \theta^{\star}\\
    =& \frac{1}{\sigma^2} V^{-1}_{t} \Phi^{\top}_{t-1} \xi_{t-1} - \lambda V^{-1}_{t} \theta^{\star}.
\end{align*}

Thus,
\begin{equation}
\label{equ:l_Vt}
    \mathbb{E}_{\theta^{\star}} \left[ \| \hat \theta_t - \theta^{\star} \|_{V_t} \right] \leq \frac{1}{\sigma^2} \mathbb{E}_{\theta^{\star}} \left[ \|  V^{-1}_{t} \Phi^{\top}_{t-1} \xi_{t-1} \|_{V_t} \right] + \lambda \mathbb{E}_{\theta^{\star}} \left[ \| V^{-1}_{t} \theta^{\star}\|_{V_t} \right],
\end{equation}
where 
\begin{equation}
\label{equ:l_Vt_p1}
    \mathbb{E}_{\theta^{\star}} \left[ \| V^{-1}_{t} \theta^{\star}\|_{V_t} \right] = \mathbb{E}_{\theta^{\star}} \left[ \sqrt{{\theta^{\star}}^{ \top} V^{-1}_{t} \theta^{\star}} \right] \leq \frac{1}{\sqrt{\lambda}} \|\theta^{\star}\|,
\end{equation}
and 
\begin{align*}
    \mathbb{E}_{\theta^{\star}} \left[ \|  V^{-1}_{t} \Phi^{\top}_{t-1} \xi_{t-1} \|_{V_t} \right] =& \mathbb{E}_{\theta^{\star}} \left[  \sqrt{\xi_{t-1}^{\top} \Phi_{t-1} V^{-1}_{t} \Phi^{\top}_{t-1} \xi_{t-1} } \right]\\
    = &\mathbb{E}_{\theta^{\star}} \left[ \mathbb{E} \left[ \left.  \sqrt{\xi_{t-1}^{\top} \Phi_{t-1} V^{-1}_{t} \Phi^{\top}_{t-1} \xi_{t-1} } \right \vert \Phi_{t-1} \right] \right]\\
    \overset{v:= \Phi^{\top}_{t-1} \xi_{t-1}}{=} &\mathbb{E}_{\theta^{\star}} \left[\mathbb{E} \left[  \left. \sqrt{v^{\top} V^{-1}_{t} v} \right \vert \Phi_{t-1}\right] \right],
\end{align*}
with $v \in \mathbb{R}^{d}$ following the distribution $\mathcal{N}(0, \sigma^2 \Phi^{\top}_{t-1} \Phi_{t-1})$ conditioned on $\Phi_{t-1}$ because the noise vector $\xi_{t-1} \mid \Phi_{t-1} \sim \mathcal{N}(0, \sigma^2 I)$. Recall $ V_{t} = \frac{1}{\sigma^2} \Phi^{\top}_{t-1} \Phi_{t-1} + \lambda I$, therefore
\begin{equation}
\label{equ:l_Vt_p2}
    \mathbb{E}_{\theta^{\star}} \left[ \|  V^{-1}_{t} \Phi^{\top}_{t-1} \xi_{t-1} \|_{V_t} \right] \leq \sigma^2 \sqrt{d}.
\end{equation}
Plugging (\ref{equ:l_Vt_p1}) and (\ref{equ:l_Vt_p2}) into (\ref{equ:l_Vt}), we have
\begin{equation}
    \mathbb{E}_{\theta^{\star}} \left[ \| \hat \theta_t - \theta^{\star} \|_{V_t} \right] \leq \sqrt{d} + \sqrt{\lambda} \|\theta^{\star}\|.
\end{equation}
Then plug the inequality above together with (\ref{equ:TS_l_Vt}) into (\ref{equ:TS_l2_ub}), we finally arrive at
\begin{equation*}
    \mathbb{E}_{E^{\theta^{\star}}} \left[ \| \tilde \theta_t \| \right] \leq  2 \| \theta^{\star} \| + 2 \sqrt{\frac{d}{\lambda}}.
\end{equation*}
\end{proof}

\subsection{Proof of Corollary \ref{cor:consecutive_diff}}
\label{prf_cor:consecutive_diff}
\begin{proof}
As shown in Lemma \ref{lmm:hp_events}, conditioned on event $E^{\theta^{\star}} \left( \frac{\delta}{2} \right)$, TS estimate $\tilde\theta_t$s are not far away from $\theta^{\star}$ simultaneously:
\begin{equation}
\label{equ:CE}
    \| \tilde\theta_t -\theta^{\star} \|_{V_{t}} \leq \beta^{\theta^{\star}}_{t} \left(\frac{\delta}{2} \right) + \alpha_{t} \left(\frac{\delta}{2} \right), \quad \forall t \in [T].
\end{equation}
Thus for $\forall t \in [T-1]$, $\tilde\theta_t$ should not be far away from $\tilde\theta_{t-1}$ with the same high probability.
From equation (\ref{equ:CE}), we have
\begin{align*}
     \| \tilde\theta_t -\theta^{\star} \|_{V_{t}} &\leq \beta^{\theta^{\star}}_{t} \left(\frac{\delta}{2} \right) + \alpha_{t} \left(\frac{\delta}{2} \right),\\
     \| \tilde\theta_{t+1} -\theta^{\star} \|_{V_{t}} &\leq \| \tilde\theta_{t+1} -\theta^{\star} \|_{V_{t+1}} \leq \beta^{\theta^{\star}}_{t+1} \left(\frac{\delta}{2} \right) + \alpha_{t+1} \left(\frac{\delta}{2} \right)
\end{align*}

By the triangle inequality of norm $\|\cdot\|_{V_{t}}$, it holds that
\begin{align*}
     \| \tilde\theta_t -\tilde\theta_{t+1} \|_{V_{t}} &\leq \| \tilde\theta_t -\theta^{\star} \|_{V_{t}} +  \| \tilde\theta_{t+1} -\theta^{\star} \|_{V_{t}}\\
     &\leq 2\beta^{\theta^{\star}}_{t+1} \left(\frac{\delta}{2} \right) + 2\alpha_{t+1} \left(\frac{\delta}{2} \right),
\end{align*}
where the last inequality holds due to the monotonicity in $\{\beta^{\theta^{\star}}_{t} \left(\frac{\delta}{2} \right)\}_{t=1}^{T}$ and $\{\alpha_{t} \left(\frac{\delta}{2} \right)\}_{t=1}^{T}$.

Therefore, we have built up the confidence ellipsoid for $\tilde\theta_t -\tilde\theta_{t+1}$ as
\begin{equation*}
    \| \tilde\theta_t -\tilde\theta_{t+1} \|_{V_{t}} \leq 2\beta^{\theta^{\star}}_{t+1} \left(\frac{\delta}{2} \right) + 2\alpha_{t+1} \left(\frac{\delta}{2} \right),
\end{equation*}
which fits into the condition of Lemma \ref{lmm:batch_update} and leads to the result that
\begin{equation*}
     \sum_{t = 1}^{T-1} \sum_{i = 1}^{M}  \left|\langle \tilde \theta_{t} -  \tilde \theta_{t+1}, x_{t,i} \rangle \right| \leq \operatorname{RGT} \left( 2\beta^{\theta^{\star}}_{T} \left(\frac{\delta}{2} \right) + 2\alpha_{T} \left(\frac{\delta}{2} \right) \right) .
\end{equation*}
\end{proof}

\subsection{Proof of Claim \ref{clm:RGT_ord}}
\label{prf_clm:RGT_ord}
\begin{proof}
\begin{itemize}
    \item $\mathbb{E}_{\theta^{\star} \sim \pi} \left[ B_f^{\theta^{\star}} \right]$.
    
    Recall that $\theta^{\star}$ is coming from the prior $\mathcal{N} \left(\mathbf{0}, \lambda^{-1} \mathbf{I} \right)$, so $\theta^{ \star(i)} \overset{\text{i.i.d.}}{\sim} \mathcal{N} \left( 0, \lambda^{-\frac{1}{2}} \right)$, then
    \begin{align*}
        &\mathbb{E} \left[ \|\theta^{\star}\| \right] \leq \sqrt{\mathbb{E} \left[ \|\theta^{\star}\|^2 \right]} = \sqrt{\frac{d}{\lambda}},\\
        &\mathbb{E} \left[  B^{\theta^{\star}}_f \right] = 2 L \cdot \mathbb{E} \left[ \|\theta^{\star}\| \right] = O \left( \sqrt{\frac{d}{\lambda}} L\right).
    \end{align*}
    
    \item $\mathbb{E}_{\theta^{\star} \sim \pi} \left[ \operatorname{RGT}  \left( \beta^{\theta^{\star}}_T\left(\frac{1}{2T}\right) + \alpha_T\left(\frac{1}{2T}\right) \right) \right]$.
    
    Recall from (\ref{def:pred_err})the definition of $\operatorname{RGT} \left( \eta_T(\delta) \right)$ as
    \begin{equation*}
        \operatorname{RGT} \left( \eta_T(\delta) \right) = \eta_T(\delta) \sqrt{ \frac{2L^2 + 2\lambda}{\lambda}} \cdot \sqrt{ dMT \log\left(\frac{\sigma^2 d \lambda + M T L^2}{\sigma^2 d \lambda}\right)}  + \eta_T(\delta) \frac{ 2L}{\sqrt{\lambda}} \cdot dM \log\left(\frac{\sigma^2 d \lambda + M T L^2}{\sigma^2 d \lambda}\right)),
    \end{equation*}
    in which only term $\eta_T(\delta)$ is $\theta^{\star}$ dependent.
    
    Also recall the definitions of $\beta^{\theta^{\star}}_T(\delta)$ and $\alpha_T(\delta)$ from (\ref{beta_t}) and (\ref{alpha_t}), we have
    \begin{align}
        &\mathbb{E}_{\theta^{\star} \sim \pi} \left[ \beta^{\theta^{\star}}_T\left(\frac{1}{2T}\right) \right] \leq  \sqrt{2 \log \left( 2T \right) + d \log\left(\frac{\sigma^2 \lambda d + T M L^2}{ \sigma^2 \lambda d}\right)} +  \sqrt{d},\\
        &\mathbb{E}_{\theta^{\star} \sim \pi} \left[ \alpha_T\left(\frac{1}{2T}\right) \right] =  2\sqrt{2 d \log \left( 2T\right) } + \sqrt{d}.
    \end{align}
    Plugging into $\eta_T(\delta) = \beta^{\theta^{\star}}_T(\delta) + \alpha_T(\delta)$, then
    \begin{align}
        \nonumber\mathbb{E}_{\theta^{\star} \sim \pi} \left[ \eta_T\left(\frac{1}{2T}\right) \right] &= \mathbb{E}_{\theta^{\star} \sim \pi} \left[ \beta^{\theta^{\star}}_T\left(\frac{1}{2T}\right) \right] + \mathbb{E}_{\theta^{\star} \sim \pi} \left[ \alpha_T\left(\frac{1}{2T}\right) \right]\\
        \nonumber&\leq \sqrt{2 \log \left( 2T \right) + d \log\left(\frac{\sigma^2 \lambda d + T M L^2}{ \sigma^2 \lambda d}\right)} + \sqrt{d} +  2\sqrt{2 d \log \left( 2T\right) } + \sqrt{d}\\
        \label{equ:eta_exp}&= O\left( \sqrt{d \log\left(\frac{\sigma^2 \lambda d + T M L^2}{ \sigma^2 \lambda d}\right)} \right).
    \end{align}
    Therefore, we bound the order of $\mathbb{E}_{\theta^{\star} \sim \pi} \left[ \operatorname{RGT}  \left( \beta^{\theta^{\star}}_T\left(\frac{1}{2T}\right) + \alpha_T\left(\frac{1}{2T}\right) \right) \right]$ by
    \begin{align*}
        &\mathbb{E}_{\theta^{\star} \sim \pi} \left[ \operatorname{RGT}  \left( \beta^{\theta^{\star}}_T\left(\frac{1}{2T}\right) + \alpha_T\left(\frac{1}{2T}\right) \right) \right] \\
        =& O\left( \frac{ L}{\sqrt{\lambda}} d \sqrt{ M T} \log\left(\frac{\sigma^2 \lambda d + T M L^2}{ \sigma^2 \lambda d}\right) \right) + O\left(\frac{L}{\sqrt{\lambda}} d^{\frac{3}{2}} M \log\left(\frac{\sigma^2 \lambda d + T M L^2}{ \sigma^2 \lambda d}\right)\right).
    \end{align*}
    \end{itemize}
 \end{proof}

\end{document}